%% file: main.tex
\documentclass{article}

\usepackage[scaled=0.85]{beramono}

\usepackage{./arxiv}

\usepackage[numbers,compress,sort]{natbib}

\usepackage[utf8]{inputenc} 
\usepackage[T1]{fontenc}    
\usepackage[table]{xcolor}  
\definecolor{light-blue}{RGB}{0, 118, 255}
\usepackage[colorlinks=true, linkcolor=light-blue, citecolor=light-blue, urlcolor=light-blue, hyperfootnotes=false]{hyperref}
\usepackage[perpage]{footmisc}
\usepackage{url}            
\usepackage{booktabs}       
\usepackage{amsfonts}       
\usepackage{amsmath}
\usepackage{amsthm}
\usepackage{thm-restate}
\usepackage{nicefrac}       
\usepackage{microtype}      
\usepackage{listings}       
\usepackage{enumitem}       
\usepackage{graphicx}
\usepackage{subcaption}
\usepackage{multirow}
\usepackage{calc}
\usepackage{cleveref}
\usepackage{pifont}
\usepackage{makecell}
\usepackage{import}

\makeatletter
\renewcommand\paragraph{\@startsection{paragraph}{4}{\z@}%
  {0pt}
  {-0.5em}
  {\normalfont\normalsize\bfseries}%
  } 
\makeatother

\theoremstyle{plain}
\newtheorem{theorem}{Theorem}[section]

\newtheorem{proposition}[theorem]{Proposition}

\theoremstyle{definition}
\newtheorem{definition}[theorem]{Definition}

\theoremstyle{remark}
\newtheorem{remark}[theorem]{Remark}

\newcommand{\R}{\mathbb{R}}

\import{./}{listing_styles.tex}

\title{\texttt{FLARE}: Verifying MILP Reformulations with LLM-Based Theorem Proving}

\author{%
  Henry Robbins \\
  Stanford University\\
  \texttt{hwr@stanford.edu} \\
  \And
  Connor Lawless \\
  Stanford University\\
  \texttt{lawlessc@stanford.edu} \\
  \And
  Madeleine Udell \\
  Stanford University\\
  \texttt{udell@stanford.edu} \\
  \And
  Ellen Vitercik \\
  Stanford University\\
  \texttt{vitercik@stanford.edu} \\
}

\begin{document}

\maketitle

\begin{abstract}
\import{./}{sections/0_abstract}
\end{abstract}

\import{./}{sections/1_introduction}
\import{./}{sections/2_related_work}
\import{./}{sections/3_formulations}
\import{./}{sections/4_reformulations}
\import{./}{sections/5_methodology}
\import{./}{sections/6_experiments}

\import{./}{sections/7_conclusion}
\import{./}{sections/acknowledgements}


{
\small
\bibliographystyle{plainnat}
\bibliography{ref}
}


\newpage
\appendix
\import{./}{appendix/1_constructive_audet_proof}

\import{./}{appendix/2_lean_formalization}
\newpage
\import{./}{appendix/3_implementation}
\newpage
\import{./}{appendix/4_prompts}
\newpage
\import{./}{appendix/5_formulation_bench}
\newpage
\import{./}{appendix/6_invalid_reformulations}
\newpage
\import{./}{appendix/7_experimental_details}
\newpage
\import{./}{appendix/8_additional_results}

\end{document}

%% file: listing_styles.tex


\definecolor{leanKeyword}{HTML}{7B1FA2}
\definecolor{leanComment}{HTML}{6A737D}
\definecolor{leanString}{HTML}{0B6623}
\definecolor{leanBg}{HTML}{F7F7F7}

\definecolor{jinjaExpr}{HTML}{00796B}
\definecolor{jinjaStmt}{HTML}{6A1B9A}
\definecolor{jinjaComment}{HTML}{6A737D}


\lstdefinelanguage{Lean4}{
  morekeywords={import,structure,where,def,theorem,lemma,example,instance,
    class,inductive,axiom,variable,open,namespace,end,if,then,else,let,in,
    do,fun,match,with,by,have,show,return,forall,exists},
  morekeywords=[2]{Type,Prop,Params,Vars,feasible,obj,paramMap,fwd,bwd,
    fwd_feas,bwd_feas,objMap,objMap_mono,fwd_obj,bwd_obj,
    MILPFormulation,MILPEquiv,StrictMono,StrictAnti},
  sensitive=true,
  morecomment=[l]{--},
  morecomment=[s]{/-}{-/},
  morestring=[b]",
  literate=
    {ℝ}{{$\mathbb{R}$}}1
    {→}{{$\to$}}1
    {∀}{{$\forall$}}1
    {∃}{{$\exists$}}1
    {∨}{{$\vee$}}1
    {∧}{{$\wedge$}}1
    {¬}{{$\neg$}}1
    {≤}{{$\leq$}}1
    {≥}{{$\geq$}}1
    {≠}{{$\neq$}}1
    {ℕ}{{$\mathbb{N}$}}1
    {ℤ}{{$\mathbb{Z}$}}1
    {↦}{{$\mapsto$}}1
    {∈}{{$\in$}}1
    {∑}{{$\sum$}}1
    {⟨}{{<}}1
    {⟩}{{>}}1
}

\lstdefinestyle{lean}{
  language=Lean4,
  basicstyle=\ttfamily\small,
  keywordstyle=\color{leanKeyword}\bfseries,
  keywordstyle=[2]\color{black},
  commentstyle=\color{leanComment}\itshape,
  stringstyle=\color{leanString},
  backgroundcolor=\color{leanBg},
  frame=single,
  rulecolor=\color{leanComment!40},
  framesep=4pt,
  xleftmargin=6pt,
  xrightmargin=6pt,
  showstringspaces=false,
  breaklines=true,
  columns=fullflexible,
  keepspaces=true,
  upquote=true,
}


\lstdefinestyle{jinja}{
  basicstyle=\ttfamily\small,
  moredelim=[s][\color{jinjaExpr}\bfseries]{\{\{}{\}\}},
  moredelim=[s][\color{jinjaStmt}\bfseries]{\{\%}{\%\}},
  moredelim=[s][\color{jinjaComment}\itshape]{\{\#}{\#\}},
  backgroundcolor=\color{leanBg},
  frame=single,
  rulecolor=\color{leanComment!40},
  framesep=4pt,
  xleftmargin=6pt,
  xrightmargin=6pt,
  breaklines=true,
  showstringspaces=false,
  tabsize=2,
  columns=fullflexible,
  keepspaces=true,
}

%% file: sections/0_abstract.tex
Mixed-Integer Linear Programming (MILP) is a fundamental tool for combinatorial optimization with extensive real-world applications. A central challenge is designing computationally efficient MILP formulations. Large Language Models (LLMs) offer new opportunities to automate the modeling process, from deriving formulations to strengthening them. Reliable automation requires robust methods for verifying that proposed formulations preserve the underlying optimization problem. However, existing approaches evaluate formulations numerically and fail to reason about general problem instances. We resolve this limitation by introducing a constructive definition of MILP reformulation that can be formalized in Lean and machine-checked. We develop \texttt{FLARE}\footnote{\texttt{FLARE} is implemented in the \texttt{milp-flare} Python package; see \url{https://flare.henryrobbins.com}.} (Formulation-Level Automated Reformulation Evaluation), a method that uses an LLM-based agent and the Lean proof assistant to verify proposed reformulations against a reference formulation. 
To evaluate our approach, we introduce FormulationBench\footnote{Download via the \texttt{formulation-bench} Python package; see \url{https://formulation-bench.henryrobbins.com}.}, a challenging dataset of 20 problems and 109 formulations. \texttt{FLARE} outperforms existing methods, with \textbf{100\%} accuracy on the NP-hard subset of FormulationBench. Furthermore, \texttt{FLARE} produces a machine-checkable certificate for every reformulation it accepts. For cases where formal guarantees are not necessary, we introduce \texttt{FLARE-NL}, a fast and cheap LLM proxy that matches \texttt{FLARE}'s accuracy but produces no certificate.\footnote{All experimental code is available at \url{https://github.com/henryrobbins/flare}.} These methods enable reliable verification in automated optimization modeling.

%% file: sections/1_introduction.tex
\section{Introduction}

Mixed-Integer Linear Programming (MILP) is a fundamental tool for combinatorial optimization with applications in scheduling \cite{floudas2005,ku2016}, planning \cite{pochet2006}, energy \cite{morais2010,knueven2020}, and chip design \cite{srinivasan2006}. A central challenge in MILP is problem formulation: translating a real-world scenario into a concrete mathematical model. Historically, problem formulation has required significant technical expertise. However, recent work has highlighted the ability of LLMs to translate natural-language problem descriptions into MILP formulations \cite{xiao2024, ahmaditeshnizi2024, astorg2025autoformulation, huang2025orlm, chen2025optichat}. The primary objective is to generate MILP formulations that faithfully represent the underlying optimization problem. 

Developing a faithful formulation is just the first stage in the modeling process. The same problem can often be expressed by multiple formulations that can lead to dramatically different solve times. In practice, experts use techniques like reformulation \cite{vanderbeck2010}, decomposition \cite{wolsey2020}, and cutting planes \cite{marchand2002a} to obtain efficient MILP formulations. LLMs offer the potential to automate this process \cite{yazdani2025, ferchtandiker2025}. More generally, we can view optimization modeling as a special case of algorithm design, where the chosen model dictates the algorithm's runtime. In this broader context, LLMs have recently been used to iteratively evolve efficient algorithms through generation and evaluation~\cite{romera-paredes2024,liu2024a,ye2024,novikov2025}. This approach has led to advances in mathematical discovery \cite{georgiev2025}, vehicle routing \cite{hottung2025,xie2026}, and system design \cite{hamadanian2025}. 

\paragraph*{Key Challenges.} A central challenge in automated algorithm design, particularly acute in the context of optimization modeling, is to ensure that an AI-generated formulation faithfully represents the original problem.
Existing approaches rely heavily on heuristics, most commonly comparing optimal objective values on a single instance \cite{ahmaditeshnizi2024, huang2025orlm, liang2026llm}. However, such checks are unreliable and can fail under simple transformations such as adding cutting planes or rescaling an objective (see \cite{zhai2025a} for a discussion).

Inspired by Karp reductions in complexity theory, \citet{zhai2025a} introduced EquivaMap, which uses LLMs to discover how solutions map between formulations. However, EquivaMap still operates at the \emph{instance level}. It validates a reformulation for a specific instance (e.g., a particular set of jobs to schedule) but not in general (e.g., any possible set of jobs to schedule). 
Such instance-level checks can miss formulation inconsistencies that arise on other problem instances. For example, we identify several cutting planes proposed by an LLM-based modeling framework called EvoCut~\cite{yazdani2025} that pass instance-level validation but remove optimal solutions for some instances. To mitigate this risk, we seek \emph{formulation-level} guarantees that hold for every problem instance.  Such guarantees require reasoning over all instances, rather than computationally evaluating a small subset of instances.

\begin{figure}[t]
\centering
\includegraphics[width=\linewidth]{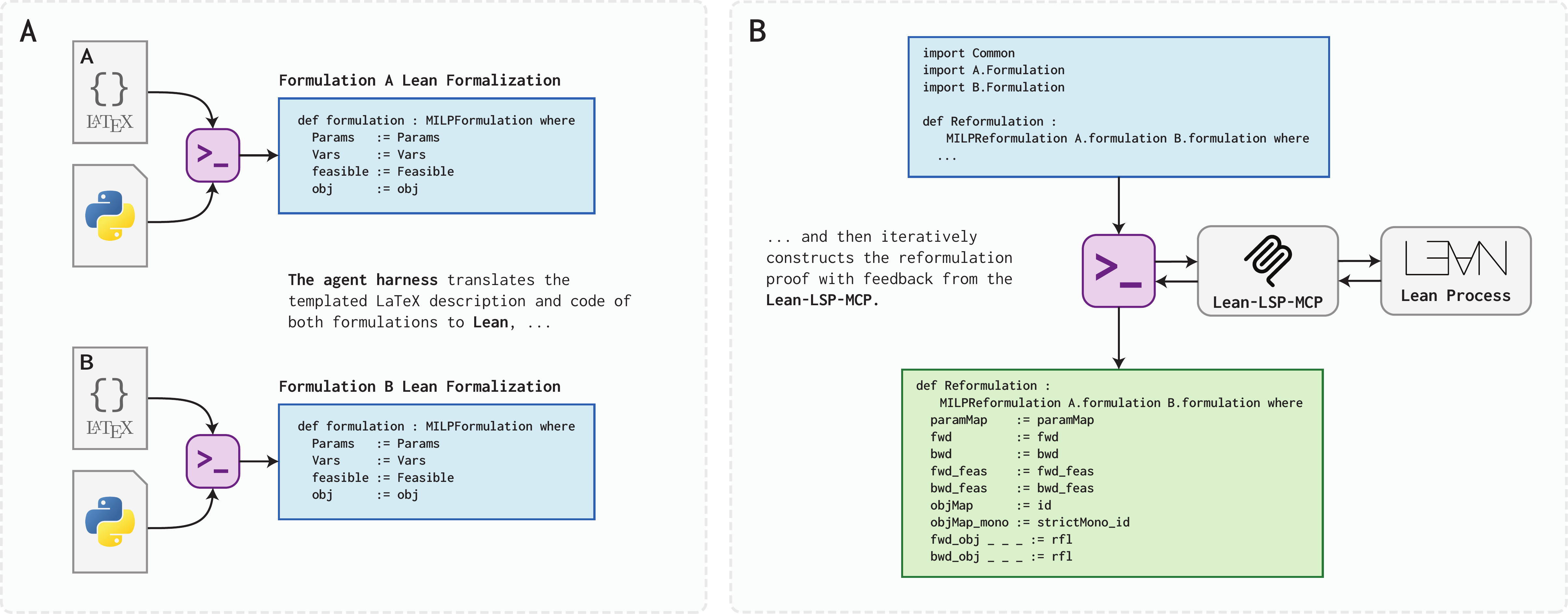}
\captionsetup{font=footnotesize, labelfont=bf}
\caption{The \texttt{FLARE} workflow. (a) An agent is given templated \LaTeX{} and Python representations of a pair of MILP formulations and their parameter mapping, and is instructed to formalize them in Lean. Combined with our Lean formalization of MILP reformulation, these outputs yield a formal claim that formulation \texttt{B} is a reformulation of \texttt{A} under the fixed parameter map. (b) The agent then attempts to construct a Lean proof of that claim. The Lean-LSP-MCP allows the agent to obtain detailed feedback from the Lean process as it develops the proof.}
\label{fig:flare}
\end{figure}

\paragraph*{Our Contributions.}
This paper introduces \texttt{FLARE}, an automated framework for validating reformulations at the formulation level. \texttt{FLARE} uses formal verification to make universal claims over problem instances machine-checkable, leveraging recent advances in automated theorem proving (ATP) to generate reformulation proofs automatically \cite{liu2026e}.

Classic notions of reformulation require reasoning about optimal solution sets and are not well-suited to ATP. To address this challenge, we introduce and formalize a constructive definition of reformulation that requires explicit mappings between parameter spaces, feasible regions, and objective values. \texttt{FLARE} instantiates this definition with autoformalized MILP formulations to obtain a formal statement that can be certified with ATP. Provided the formalizations are faithful, a certificate proves a reformulation is valid for all problem instances. 
If ATP fails to produce a certificate, \texttt{FLARE} does not certify the reformulation as valid.
\texttt{FLARE} is the first automated approach to produce verifiable reformulation certificates, enabling trustworthy optimization modeling with LLMs.

We also introduce \texttt{FLARE-NL}, a fast and cheap proxy that prompts a frontier reasoning model with the same definition. The two methods serve different regimes. A \texttt{FLARE} certificate on faithful formalizations admits no false positives; this property is critical in settings with zero tolerance for error (e.g., energy~\cite{knueven2020}). \texttt{FLARE-NL} is faster and cheaper but offers no such guarantee, making it the natural choice in non-critical settings or as a screening heuristic prior to \texttt{FLARE}.

Our contributions can be summarized as follows:

\begin{itemize}[topsep=2pt,parsep=0pt,leftmargin=2em]
\item \textbf{Constructive definition of MILP reformulation.} We introduce a constructive definition of MILP reformulation that is amenable to formal verification and ATP, enabling the first automated method for certifying reformulations at the formulation level.
\item \textbf{AI for formulation verification.} 
We develop \texttt{FLARE}, a framework that combines LLM-based autoformalization with ATP to generate machine-checkable reformulation certificates, and \texttt{FLARE-NL}, an LLM proxy that trades formal guarantees for improved speed and cost.
\item \textbf{Realistic benchmark for reformulation.}
We introduce FormulationBench, a benchmark dataset of 20 problems and 109 formulations capturing realistic modeling transformations. Empirically, we show both \texttt{FLARE} and \texttt{FLARE-NL} outperform existing methods, achieving \textbf{100\%} accuracy on the NP-hard subset of FormulationBench.
\item \textbf{Demonstrated need for formal proofs.} 
We apply \texttt{FLARE} to prior AI-driven MILP transformations, revealing 5 invalid cutting planes proposed by EvoCut \cite{yazdani2025} and 4 reformulations proposed in \citet{ferchtandiker2025}: these ``equivalent'' formulations are wrong or omit necessary assumptions (Appendix \ref{sec:invalid-reformulations}).
\end{itemize}

%% file: sections/2_related_work.tex
\section{Related Work}

An efficient MILP formulation can decrease the solve time by orders of magnitude \cite{conforti2014integer}. To improve solve time, experts apply transformations such as lifting \cite{balas1979disjunctive}, change of variables \cite{wolsey2020}, and cutting planes \cite{van1987solving} that strengthen the formulation while preserving optimal solutions (see \cite{liberti2009} for a broader discussion). We study the complementary problem of \emph{verifying} such reformulations, particularly in emerging settings where formulations are generated or modified by LLMs. This perspective connects three lines of research: (i) LLM-based systems that generate and refine MILP formulations, (ii) methods for verifying reformulations, and (iii) autoformalization and automated theorem proving.

\paragraph*{LLMs for MILP Modeling.}
A growing body of work has applied LLMs to automate the translation of natural language problem descriptions into MILP formulations. Early efforts focused on natural language processing (NLP) pipelines for structured extraction \cite{ramamonjison23}, while more recent \emph{optimization copilots} use LLM-based multi-agent systems \cite{ahmaditeshnizi2024, ahmaditeshnizi2025, mostajabdaveh2024, xiao2024, liang2026llm, astorg2025autoformulation, drossman2026let} or fine-tuning \cite{huang2025orlm, jiangllmopt} to handle complex, multi-constraint problems. These capabilities have been applied across domains including supply chain management \cite{li2023a}, diagnosing infeasible models \cite{chen2024diagnosing, chen2025optichat}, and scheduling \cite{lawless2024want}. Beyond formulating a \textit{correct} model, recent work has also explored the use of LLMs to generate \textit{stronger} MILP formulations via cutting planes \cite{yazdani2025}, reformulations \cite{ferchtandiker2025}, or better solver configuration \cite{lawless2025llms}. However, these approaches rely on instance-level validation. We demonstrate examples where this limitation leads to invalid cutting planes and reformulations.

\paragraph*{Automatic Reformulation Checking Methods.}
Determining whether one formulation is a reformulation of another is central to evaluating LLM-generated MILP formulations. Existing approaches largely operate at the instance level. Canonical accuracy checks for a direct mapping between declarations (i.e., constraints and objectives) \cite{ramamonjison23}, while execution accuracy compares optimal objective values after solving both formulations \cite{ahmaditeshnizi2024, huang2025orlm, liang2026llm}. More recently, \textit{EquivaMap} \cite{zhai2025a} used LLMs to infer variable mappings between formulations and validate them by mapping optimal solutions on a specific instance. Structural approaches represent MILPs as bipartite graphs and measure similarity via graph isomorphism or edit distance \cite{xing2024, steever2024graph, wang2024, wang2025orgeval}, avoiding the need to solve the optimization problem directly. Unlike existing approaches, \texttt{FLARE} verifies reformulations at the formulation-level by producing machine-checkable certificates that hold across all problem instances.

\paragraph*{Autoformalization and Automated Theorem Proving.}
Autoformalization translates natural language into formal representations~\cite{wu2022a,gao2025,wang2025e,liu2025f,jana2026}, while automated theorem proving (ATP) generates machine-checkable proofs for formal statements~\cite{lin2025, ren2025b, varambally2025, chen2025d, axiom2026, logicalintelligence2026, liu2026e, requena2026}. Recent advances in LLMs have significantly improved both tasks, with agentic frameworks combining (fine-tuned) language models and proof assistants (e.g., Lean) to formalize and solve complex mathematical problems. Most recently, simple agentic harnesses have been shown to be competitive ATP methods~\cite{requena2026,liu2026e}. We adopt an LLM-based agent as the ATP method in \texttt{FLARE}  and introduce a constructive definition of reformulation to make proving reformulations tractable.

%% file: sections/3_formulations.tex
\section{Formalizing Formulations}
\label{sec:formulations}

We begin by formally defining a \emph{MILP formulation}. In particular, we distinguish between a \emph{formulation} and the parameter (or data) values that specify a particular \emph{instance} of the problem \cite{fourer2003}. Take the traveling salesman problem (TSP) as an example. A TSP formulation is defined on an abstract set of cities. A TSP instance is one such set of cities. Instantiating the formulation with an instance yields a concrete MILP to be solved.

\subimport{definitions/}{formulation}

\paragraph*{Running Example.}
\label{sec:tsp}
\subimport{formulations/}{tsp}

Both formulations define decision variables $x_{ij} \in \{0,1\}$ to indicate if the edge $(i,j)$ is used in the tour. To prevent subtours, \eqref{eq:dfj} uses an exponential family of subtour-elimination constraints, while \eqref{eq:mtz} uses a polynomial family of constraints with auxiliary position variables $u_i \in \R$ for each node. We denote these formulations as $\mathcal{M}_{\text{DFJ}}$ and $\mathcal{M}_{\text{MTZ}}$, respectively. Both formulations faithfully represent the TSP, but have notable differences that affect solve times. $\mathcal{M}_{\text{DFJ}}$ is a \emph{stronger} formulation than $\mathcal{M}_{\text{MTZ}}$ in that it admits fewer fractional solutions in the linear relaxation. Despite the exponential size of $\mathcal{M}_{\text{DFJ}}$, it can be solved efficiently with constraint generation.

With this formal definition established, we now consider multiple notions of MILP reformulation and introduce a constructive definition amenable to formalization and ATP.

%% file: definitions/formulation.tex
\begin{restatable}[]{definition}{formulation}
\label{def:formulation}
A MILP \emph{formulation} $\mathcal{M}$ is a tuple ${\mathcal{M} = (\mathcal{P}, \mathcal{F}, f_0)}$ with parameter space $\mathcal{P}$, feasible region $\mathcal{F}(p) \subseteq \R^{n(p)}$, and objective function $f_0$. For instance $p \in \mathcal{P}$, the feasible region $\mathcal{F}(p)$ is defined by $m(p)$ linear constraints, $f_i(\cdot ; p) : \R^{n(p)} \to \R$ for all $i \in [m(p)]$. The first $k(p) \leq n(p)$ variables are integers. The feasible region is
$$\mathcal{F}(p) = \{x \in \mathbb{Z}^{k(p)} \times \mathbb{R}^{n(p)-k(p)}~|~f_i(x;p) \leq 0 \; \forall i\in[m(p)]\}.$$
When the instance $p$ is clear from context, we use $n$, $m$, and $k$ instead of the parameterized forms. The objective is to minimize\footnote{We write all formulations as minimization problems; maximization objectives can be converted by negating the objective.} the linear function $f_0(\cdot ; p) : \R^{n(p)} \to \R$. A formulation $\mathcal{M}$ is \emph{instantiated} with an instance $p \in \mathcal{P}$. We denote an instantiated formulation as ${\mathcal{M}(p) = (\mathcal{F}(p), f_0(p))}$.
\end{restatable}

%% file: formulations/tsp.tex
The traveling salesman problem (TSP) aims to find the shortest tour in a graph that visits every node exactly once. A TSP instance is a weighted fully-connected graph $G=(V,E,w)$ on $n = |V| \geq 2$ nodes with edge weights $w_{ij}$. We write $\mathcal{P}_{\text{TSP}}$ for the set of all such graphs. Two common MILP formulations for the TSP are the Dantzig-Fulkerson-Johnson (DFJ) \cite{dantzig1954} and Miller-Tucker-Zemlin (MTZ) \cite{miller1960} formulations:

\noindent\scriptsize%
\begin{minipage}[t]{0.48\textwidth}
\vspace{-8pt}
\begin{equation}
\begin{aligned}
\min\;&\sum_{(i,j)\in E} w_{ij}\,x_{ij}\\[3pt]
\text{s.t.}\;
&\sum_{j\in V\setminus\{i\}} x_{ij}=1 && \forall i\in V\\
&\sum_{i\in V\setminus\{j\}} x_{ij}=1 && \forall j\in V\\
&\sum_{i\in S}\sum_{j \in S} x_{ij}\le|S|-1 && \forall S\subset V,\;2\le|S|\le n{-}1\\
&x_{ij}\in\{0,1\} && \forall(i,j)\in E
\end{aligned}
\tag{DFJ}\label{eq:dfj}
\end{equation}
\end{minipage}%
\hfill%
\begin{minipage}[t]{0.48\textwidth}
\vspace{-8pt}
\begin{equation}
\begin{aligned}
\min\;&\sum_{(i,j)\in E} w_{ij}\,x_{ij}\\[3pt]
\text{s.t.}\;
&\sum_{j\in V\setminus\{i\}} x_{ij}=1 &&\forall i\in V\\
&\sum_{i\in V\setminus\{j\}} x_{ij}=1 &&\forall j\in V\\
&u_i-u_j+n\,x_{ij}\le n-1 &&\forall i,j\in V\setminus\{1\},\; i\neq j\\
&u_1=1 \\
&2\le u_i\le n &&\forall i\in V \setminus\{1\}\\
&x_{ij}\in\{0,1\} &&\forall(i,j)\in E
\end{aligned}
\tag{MTZ}\label{eq:mtz}
\end{equation}
\end{minipage}\normalsize
\vspace{4pt}

%% file: sections/4_reformulations.tex
\section{Formalizing Reformulations}

Until now, we have used the term \emph{reformulation} informally. It has an intuitive operational meaning: $\mathcal{M}'$ is a reformulation of $\mathcal{M}$ if one can map an instance $\mathcal{M}(p)$ to an instance $\mathcal{M}'(p')$, solve $\mathcal{M}'(p')$, and efficiently recover an optimal solution to $\mathcal{M}(p)$. Importantly, this is a \emph{formulation-level} claim: the construction must work for \emph{all} problem instances $p \in \mathcal{P}$.

In Section~\ref{sec:existing-defs}, we review an existing notion of reformulation capturing this intuition and discuss why this definition is difficult to verify computationally. This limitation has led to proxies for reformulation that are easier to check computationally but lack formulation-level guarantees (Section \ref{sec:proxy-defs}). Our key idea is to utilize tools from formal verification to make formulation-level claims machine-checkable. To this end, we propose a constructive definition of reformulation that is amenable to formalization and tractable for ATP (Section \ref{sec:constructive-def}).

\subsection{Existing Definition}
\label{sec:existing-defs}

\citet{audet1997} capture our intuitive notion of reformulation with a complexity-theoretic definition inspired by polynomial-time Turing reductions \cite{garey2009}.

\subimport{definitions/}{audet}

\begin{remark}
To avoid trivial or vacuous cases, we focus on settings in which solving each
formulation is NP-hard and each formulation has at least one feasible instance. The polynomial-time restriction on the optimal-solution mapping excludes trivial mappings that solve $\mathcal{M}(p)$ directly.
\end{remark}

\paragraph*{TSP Example.}
The formulation $\mathcal{M}_{\text{MTZ}}$ is an Audet reformulation of $\mathcal{M}_{\text{DFJ}}$. The mapping $\Phi_{\mathrm{p}}$ is the identity as both formulations use the same parameter space $\mathcal{P}_{\text{TSP}}$. Given any optimal solution of $\mathcal{M}_{\text{MTZ}}(\Phi_{\mathrm{p}}(p))$, we can obtain an optimal solution to $\mathcal{M}_{\text{DFJ}}(p)$ by mapping $x_{ij}$ to itself and dropping the position variables $u_i$, a linear time operation.

This definition captures a formulation-level notion of reformulation as desired. However, it is difficult to verify computationally since it quantifies over all instances $p \in \mathcal{P}$. Furthermore, ATP must identify a mapping that sends every optimal solution of $\mathcal{M}'(\Phi_{\mathrm{p}}(p))$ to an optimal solution of $\mathcal{M}(p)$. Our constructive definition of reformulation in Section \ref{sec:constructive-def} is designed to lighten the demands on ATP.

\subsection{Proxy Definitions and Limitations}
\label{sec:proxy-defs}

Because formulation-level verification is difficult, existing methods use instance-level proxies. For a fixed pair of instances, a proxy checks whether $\mathcal{M}'(p')$ is a reformulation of $\mathcal{M}(p)$. The simplest proxy solves $\mathcal{M}(p)$ and $\mathcal{M}'(p')$ and compares the optimal objective values~\cite{ahmaditeshnizi2024, huang2025orlm, liang2026llm}. \citet{zhai2025a} propose a stronger proxy, Quasi-Karp equivalence\footnote{Note that this definition is directional despite the \emph{equivalence} terminology.}, inspired by Karp reductions~\cite{karp1972} in complexity theory.

\begin{definition}[Quasi-Karp Equivalence~\cite{zhai2025a}]
\label{def:quasi-karp-equivalence}
Let $\mathcal{M}(p)$ and $\mathcal{M}'(p')$ be two instantiated formulations. We say $\mathcal{M}'(p')$ is \textit{Quasi-Karp equivalent} to $\mathcal{M}(p)$
if there exists an algorithm $\mathcal{A}(\mathcal{M}(p),\mathcal{M}'(p'))$ that produces a mapping $f$ such that:
\begin{itemize}[topsep=2pt,parsep=0pt,leftmargin=2em]
    \item If $x^*$ is an optimal solution to $\mathcal{M}'(p')$, then $f(x^*)$ is an optimal solution to $\mathcal{M}(p)$,
    \item $f$ can be computed in polynomial time, and
    \item $\mathcal{A}(\mathcal{M}(p),\mathcal{M}'(p'))$ runs in polynomial time for all $p \in \mathcal{P}$ and $p' \in \mathcal{P}'$.
\end{itemize}
\end{definition}

\textit{Quasi-Karp equivalence} is an instance-level analogue of Audet reformulation. \citet{zhai2025a} implement this idea in EquivaMap, where an LLM proposes a mapping $f$, which is then validated on the particular solved instance. This mapping-based view is more expressive than comparing the objective values alone, allowing EquivaMap to handle some objective transformations and solution mappings on a fixed instance. However, it still does not certify at the formulation level (across all instances), and its effectiveness depends on the class of candidate mappings \footnote{To ensure $f$ is computable in polynomial time, EquivaMap restricts $f$ to be linear.} for $f$ and the LLM's ability to find the map.

\paragraph*{Pitfalls of Instance-Level Verification.} 
The distinction between instance- and formulation-level verification is critical in automated optimization modeling. 
A generated formulation is meant to be reused on unseen problem data. An instance-level check can validate a reformulation that is invalid at the formulation level, resulting in errors on unseen instances. For example, EvoCut~\cite{yazdani2025} uses LLMs to propose cutting planes (cuts) for MILP formulations (Appendix \ref{sec:evocut}). A valid cut for formulation $\mathcal{M}$ may remove feasible points in the linear relaxation of $\mathcal{M}$ while retaining all integer-feasible points. The EvoCut method validates candidate cuts with a simple execution-based proxy. As such, the proposed cuts may be invalid on unseen instances. The proposed cut \eqref{eq:tsp-v1-ec3} illustrates this risk (Proposition \ref{prop:tsp-v1-ec3}). This cut eliminates triangles involving the first node. This is acceptable for an instance with $n > 3$ nodes. However, consider the 3-node instance depicted in Figure \ref{fig:tsp} with feasible tour $1 \to 2 \to 3 \to 1$. The only feasible tour is eliminated by the cut.

\begin{figure}[h]
\centering
\begin{subfigure}[b]{0.33\linewidth}
  \phantomsubcaption\label{fig:tsp}
\end{subfigure}
\begin{subfigure}[b]{0.66\linewidth}
  \phantomsubcaption\label{fig:constructive-reformulation}
\end{subfigure}
\includegraphics[width=0.80\linewidth]{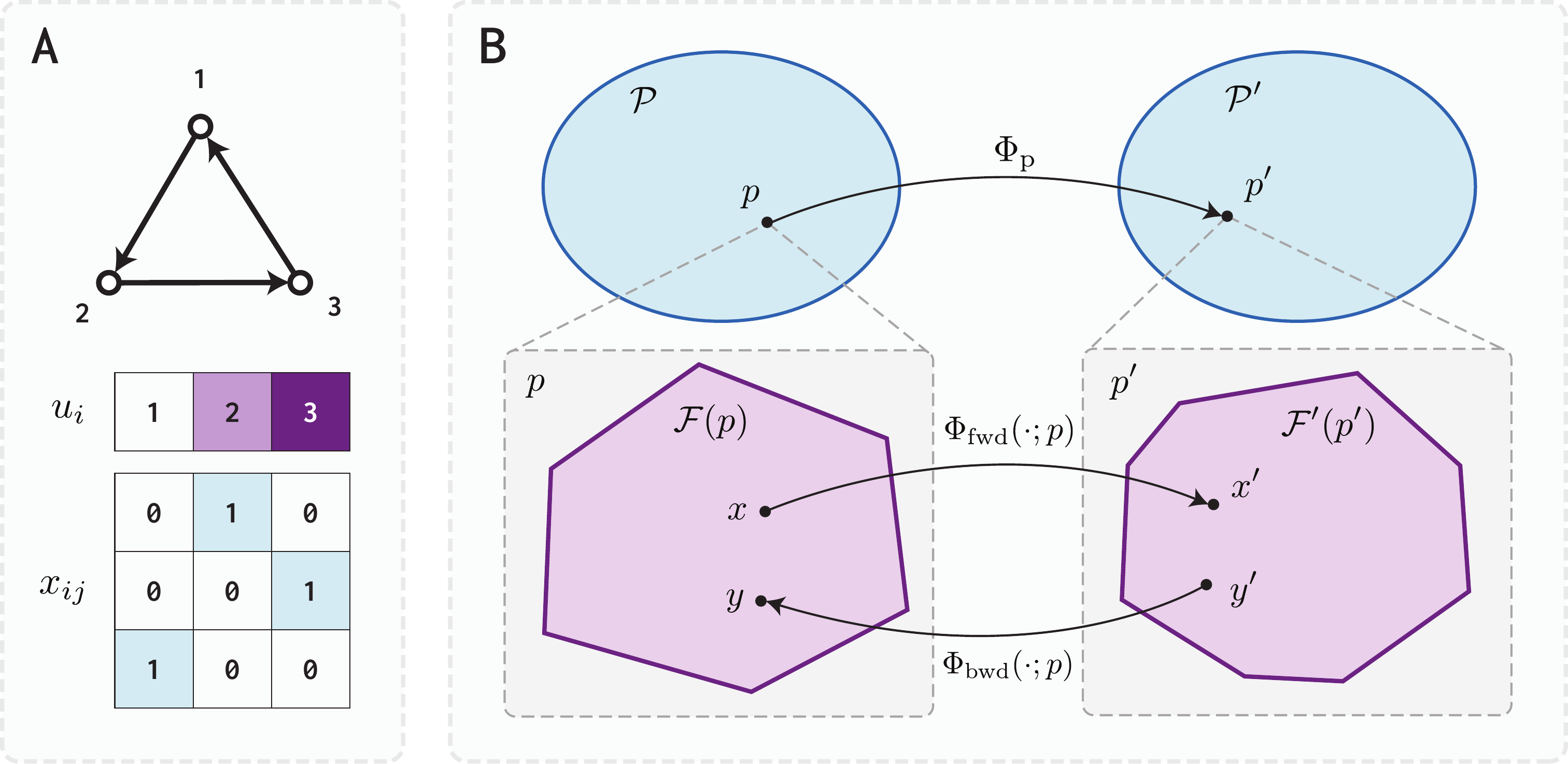}
\captionsetup{font=footnotesize, labelfont=bf}
\caption{(a) A feasible tour $1 \to 2 \to 3 \to 1$ on a 3-node TSP instance with the MTZ variables $x_{ij}$ and $u_i$. (b) A graphical depiction of the reformulation construction $\Phi(\mathcal{M}, \mathcal{M}')$. The parameter mapping $\Phi_{\mathrm{p}}$ yields a pair of feasible regions $\mathcal{F}(p)$ and $\mathcal{F}'(p')$ related by the forward and backward maps such that objective ordering is preserved.}
\end{figure}

\subsection{Constructive Definition}
\label{sec:constructive-def}

To overcome the limitations of instance-level reformulation proxies, we use formal verification to machine-check reformulations at the formulation level. Audet reformulation can be formalized and machine-checked, but it requires ATP to identify a mapping that sends every optimal solution of $\mathcal{M}'(\Phi_{\mathrm{p}}(p))$ to an optimal solution of $\mathcal{M}(p)$. To ease the burden on ATP, we propose a stronger, constructive definition of reformulation that requires explicit forward and backward mappings between feasible regions that preserve objective ordering, together with a strictly increasing objective map (see Figure~\ref{fig:constructive-reformulation}). These conditions satisfy the Audet reformulation (Proposition \ref{prop:constructive-audet}) but do not require ATP to certify the optimality of solutions in the image of the solution mapping.

\subimport{definitions/}{constructive}

\begin{remark}
\label{rem:poly-time}
Like Audet reformulation, we use this definition for settings in which solving each formulation is NP-hard and each formulation has at least one feasible instance.
\end{remark}

This definition is designed to capture common MILP transformations (e.g., lifting, rescaling, and substitution) and modeling choices (e.g., exponential constraints like the TSP \eqref{eq:dfj} formulation). A related definition by Sherali requires a much stronger order-preserving bijective mapping between the two feasible regions that excludes common transformations (e.g., lifting)~\cite{liberti2009}.

We next relate constructive reformulation to the two notions above. We prove that a \emph{constructive reformulation} is an Audet reformulation and, as an immediate corollary at the instance level, constructive reformulation also implies Quasi-Karp equivalence. We defer the proof to Appendix \ref{sec:proofs}.

\begin{restatable}[]{proposition}{constructiveaudet}
\label{prop:constructive-audet}
Let $\mathcal{M}$ and $\mathcal{M}'$ be formulations. If $\mathcal{M}'$ is a \emph{constructive reformulation} of $\mathcal{M}$, then $\mathcal{M}'$ is an Audet reformulation of $\mathcal{M}$. Furthermore, for any instance $p \in \mathcal{P}$ with $p' = \Phi_{\mathrm{p}}(p)$, if $\mathcal{M}'(p')$ is a \emph{constructive reformulation} of $\mathcal{M}(p)$, then $\mathcal{M}'(p')$ is \emph{Quasi-Karp equivalent} to $\mathcal{M}(p)$.
\end{restatable}

Our definition is stronger than Audet's because it requires mappings for \emph{all} feasible points, not only optima. This stronger notion of reformulation provides a more structured claim that reduces the proof burden on ATP while remaining satisfied by common MILP transformations. For the remainder of the paper, we use reformulation to mean \emph{constructive reformulation}.

%% file: definitions/audet.tex
\begin{definition}[Audet Reformulation~\cite{audet1997}]
\label{def:audet}
Let $\mathcal{M}$ and $\mathcal{M}'$ be two formulations with parameter spaces $\mathcal{P}$ and $\mathcal{P}'$.
$\mathcal{M}'$ is an \emph{Audet reformulation} of $\mathcal{M}$ if there exists a mapping $\Phi_{\mathrm{p}} : \mathcal{P} \to \mathcal{P}'$ such that, for any instance $p \in \mathcal{P}$: if $\mathcal{M}(p)$ has an optimal solution, then $\mathcal{M}'(\Phi_{\mathrm{p}}(p))$ also has an optimal solution, and every optimal solution to $\mathcal{M}'(\Phi_{\mathrm{p}}(p))$ can be mapped back to an optimal solution of $\mathcal{M}(p)$ in polynomial time.
\end{definition}

%% file: definitions/constructive.tex
\begin{definition}[Constructive Reformulation]
\label{def:constructive}
Let $\mathcal{M} = (\mathcal{P}, \mathcal{F}, f_0)$ and $\mathcal{M}' = (\mathcal{P}', \mathcal{F}', f'_0)$ be formulations. A \emph{reformulation construction} from $\mathcal{M}$ to $\mathcal{M}'$ is a tuple $\Phi(\mathcal{M}, \mathcal{M}') = (\Phi_{\mathrm{p}},\, \Phi_{\text{fwd}},\, \Phi_{\text{bwd}},\, \Phi_{\text{obj}})$ consisting of:
\begin{itemize}[topsep=2pt,parsep=0pt,leftmargin=2em]
    \item a parameter mapping $\Phi_{\mathrm{p}} : \mathcal{P} \to \mathcal{P}'$,
    \item a forward mapping $\Phi_{\text{fwd}}(\cdot;p) : \mathbb{R}^{n(p)} \to \mathbb{R}^{n'(\Phi_{\mathrm{p}}(p))}$,
    \item a backward mapping $\Phi_{\text{bwd}}(\cdot;p) : \mathbb{R}^{n'(\Phi_{\mathrm{p}}(p))} \to \mathbb{R}^{n(p)}$ computable in polynomial time,
    \footnote{Polynomial time is measured under fixed binary encodings of parameters and rational/integer assignments: $\Phi_{\mathrm{bwd}}$ is polynomial-time computable if a single deterministic algorithm, given $p, \Phi_{\text{p}}(p),$ and $x'$, outputs $\Phi_{\mathrm{bwd}}(x';p)$ in time polynomial in the total input bit length. Supplying $\Phi_{\text{p}}(p)$ as input means this condition does not impose
a polynomial-time requirement on $\Phi_{\text{p}}$.} and
    \item an objective mapping $\Phi_{\text{obj}} : \mathbb{R} \to \mathbb{R}$.
\end{itemize}
$\mathcal{M}'$ is a \emph{constructive reformulation} of $\mathcal{M}$ if there exists a reformulation construction satisfying the following conditions for every instance $p \in \mathcal{P}$, with $p' = \Phi_{\mathrm{p}}(p)$:
\begin{itemize}[topsep=2pt,parsep=0pt,leftmargin=2em]
    \item \textbf{Forward feasibility.} For all $x \in \mathcal{F}(p)$: $\Phi_{\text{fwd}}(x;p) \in \mathcal{F}'(p')$.
    \item \textbf{Backward feasibility.} For all $x' \in \mathcal{F}'(p')$: $\Phi_{\text{bwd}}(x';p) \in \mathcal{F}(p)$.
    \item \textbf{Strictly monotone objective mapping.} $\Phi_{\text{obj}}$ is strictly monotonically increasing.
    \item \textbf{Objective preservation.} (1) For all ${x \in \mathcal{F}(p)}$, $f_0'(\Phi_{\text{fwd}}(x;p);p') = \Phi_{\text{obj}}(f_0(x;p))$, and (2) for all ${x' \in \mathcal{F}'(p')}$, $f_0'(x';p') = \Phi_{\text{obj}}(f_0(\Phi_{\text{bwd}}(x';p);p))$.
\end{itemize}
We say $\mathcal{M}'(p')$ is a \emph{constructive reformulation} of $\mathcal{M}(p)$ if the conditions hold for $p \in \mathcal{P}$ and $p' \in \mathcal{P}'$.
\end{definition}

%% file: sections/5_methodology.tex
\section{Methodology}
\label{sec:methodology}

This section operationalizes our reformulation definition into an automated verification pipeline. We first formalize our definition in Lean and then introduce \texttt{FLARE}, which uses ATP to generate machine-checkable certificates. Finally, we introduce \texttt{FLARE-NL}, a fast and cheap LLM proxy that reasons about reformulations using only natural language rather than Lean.

\paragraph*{Formalization.}
Our reformulation definition is a formulation-level claim quantified over every instance $p \in \mathcal{P}$. To make it machine-checkable, we encode the claim in Lean, which reasons symbolically over arbitrary instances rather than testing concrete data. We formalize our definition in Lean~\cite{moura2021}, chosen for its mature Mathlib library and extensive ATP tooling (see Appendix \ref{sec:implementation}). We provide Lean formalizations of \emph{formulation} (Definition \ref{def:formulation}) and \emph{reformulation} (Definition \ref{def:constructive}) in Appendix \ref{sec:lean-formalization}. 
We assume the input formulations are well-formed MILPs (i.e., have linear constraints and objective, and are defined over real-valued decision variables). Accordingly, our Lean formalization does not explicitly encode these conditions. The only part of Definition~\ref{def:constructive} we omit from the Lean proof obligation is the polynomial-time requirement on $\Phi_{\mathrm{bwd}}$. Formalizing computational complexity in Lean would add substantial proof burden, and in our experiments, the backward map in each construction was always computable in polynomial time.

\paragraph*{\texttt{FLARE}.}
\texttt{FLARE} determines if $\mathcal{M}'$ is a reformulation of $\mathcal{M}$ under a fixed parameter mapping $\Phi_{\mathrm{p}}$ in the following stages: (1) autoformalize $\mathcal{M}$, $\mathcal{M}'$, and $\Phi_{\mathrm{p}}$ into Lean, (2) attempt to prove $\mathcal{M}'$ is a reformulation of $\mathcal{M}$ under $\Phi_{\mathrm{p}}$ using ATP, and (3) check if a reformulation certificate was successfully generated by ATP. In the first two stages, we utilize the Claude Code\footnote{Alternative agent harnesses are considered in Appendix \ref{sec:addtional-results}.} agent harness for autoformalization and ATP. Existing ATP methods, like the Numina-Lean-Agent~\cite{liu2026e}, demonstrate the use of Claude Code as a competitive ATP method when combined with the Lean-LSP-MCP~\cite{breen2025,liu2026e,ju2026a}. We further augment Claude Code with custom agent skills for handling MILP formulations in Lean. In the final stage, we check if a proof was generated and ensure it compiles without using \texttt{sorry}\footnote{The \texttt{sorry} tactic can be used in Lean proofs to skip a proof obligation. Any proof relying on \texttt{sorry} is hence incomplete.} and without introducing new axioms. This workflow is summarized in Figure \ref{fig:flare} and additional implementation details can be found in Appendix \ref{sec:implementation}.

\paragraph*{Limitations.}
\label{sec:flare-limitations}
\texttt{FLARE}'s guarantee is conditional on faithful formalization. If the formulations and parameter mapping are formalized correctly, then a Lean-verified certificate proves the reformulation claim under our definition. However, failure to generate a certificate is inconclusive. We highlight three limitations.

\begin{itemize}[topsep=2pt,parsep=0pt,leftmargin=1em]
    \item \textbf{No certificate of invalidity.} If $\mathcal{M}'$ is not a reformulation of $\mathcal{M}$ under $\Phi_{\mathrm{p}}$, \texttt{FLARE} does not prove invalidity; the ATP component simply refuses to certify validity of the reformulation. Moreover, invalidity is scoped to the fixed mapping: it is not evidence that no parameter mapping makes $\mathcal{M}'$ a reformulation of $\mathcal{M}$.

    \item \textbf{Dependence on faithful formalization.} If the formalizations are unfaithful to $\mathcal{M}$, $\mathcal{M}'$, or $\Phi_{\mathrm{p}}$, \texttt{FLARE} may certify a statement that differs from the intended claim. We conduct a human audit of our dataset (Section \ref{sec:formulation-bench}) and an LLM-as-a-judge audit of every formalization produced during our main \texttt{FLARE} experiment. A selection of \texttt{FLARE} formalizations were human-reviewed to check the correctness of the LLM judge. Both audits observe no instance of this failure mode. Furthermore, this concern can be eliminated by deterministically translating from a modeling standard such as AMPL~\cite{fourer2003} or MathOptInterface~\cite{legat2021} to Lean.
    
    \item \textbf{False negatives from theorem proving.} If the LLM is unable to identify an appropriate reformulation construction or the proof requires extensive work in Lean to formalize, the agent may exit early, producing a false negative. Ongoing efforts to formalize computer science foundations in Lean with CSLib may reduce this burden on ATP.
    
\end{itemize}

\paragraph*{\texttt{FLARE-NL}.} Generating a reformulation certificate using \texttt{FLARE}  can take 5-10 minutes. As a fast and cheap proxy, \texttt{FLARE-NL} prompts a frontier reasoning model with templated \LaTeX{} descriptions of two formulations, $\mathcal{M}$ and $\mathcal{M}'$, along with the parameter mapping $\Phi_{\mathrm{p}}$, and asks if $\mathcal{M}'$ is a reformulation of  $\mathcal{M}$ under $\Phi_{\mathrm{p}}$ (see Appendix \ref{sec:prompts} for full prompt). \texttt{FLARE-NL} differs from the naive-LLM baseline of \citet{zhai2025a} in three ways: (1) the prompt includes our definition of reformulation, (2) the prompt explicitly states all formulation assumptions and instructs not to introduce new ones, and (3) we utilize reasoning and structured output. Unlike \texttt{FLARE}, \texttt{FLARE-NL} lacks any verifiable guarantees on its output. It can serve as a screening heuristic before \texttt{FLARE}, but it is not a proof method.

%% file: sections/6_experiments.tex
\section{Experiments}
\label{sec:experiments}

To evaluate if formulation-level verification catches failures missed by instance-level proxies, we introduce FormulationBench, an extension of the EquivaFormulation dataset \cite{zhai2025a} containing more challenging formulation-level cases. We evaluate \texttt{FLARE} and \texttt{FLARE-NL} against established baselines and conduct an ablation study on \texttt{FLARE-NL}.\footnote{All experimental code is available at \url{https://github.com/henryrobbins/flare}.} Additional experimental details are provided in Appendix \ref{sec:experimental-details}.

\subsection{FormulationBench}
\label{sec:formulation-bench}

We introduce FormulationBench, a benchmark designed to test formulation-level reformulation reasoning. FormulationBench extends EquivaFormulation~\cite{zhai2025a} with candidate cuts proposed by EvoCut~\cite{yazdani2025} and a collection of eight MILP formulation pairs from~\citet{ferchtandiker2025}. These formulation pairs are more challenging than those in EquivaFormulation, requiring reasoning about general cutting plane families and meaningfully different modeling techniques. The dataset contains 20 problems, 109 formulations, and 89 formulation pairs with 63 positive and 26 negative examples (Appendix \ref{sec:formulation-bench-details}). Following Remark \ref{rem:poly-time}, we restrict our evaluation to the subset of 16 NP-hard problems containing 54 formulation pairs. FormulationBench is available via the \texttt{formulation-bench} Python package.\footnote{Documentation is available at \url{https://formulation-bench.henryrobbins.com}}

FormulationBench is organized similarly to the NLP4LP dataset \cite{ahmaditeshnizi2024}; JSON files contain formulation descriptions and problem instance data. Relative to NLP4LP, FormulationBench introduces two notable extensions. First, we explicitly record implicit assumptions necessary to prove reformulation validity. We flag assumptions implicit in the source dataset to facilitate our ablation study in Section \ref{sec:ablation}. Second, we formalize each formulation in Lean (see Appendix \ref{sec:lean-formalization}) and include ground-truth reformulation certificates for each valid reformulation pair. The nine TSP formulations were formalized by hand and used as reference examples when autoformalizing the remainder with Claude Code. All 109 Lean formalizations were then reviewed line-by-line by a PhD student specializing in operations research, who found no autoformalization errors. While preparing this dataset, we verified prior AI-generated MILP reformulations and found 5 invalid cutting planes proposed by EvoCut \cite{yazdani2025} and 4 invalid reformulations proposed in \citet{ferchtandiker2025} (see Appendix \ref{sec:invalid-reformulations}).

\subsection{Baseline Comparison}
\label{sec:baseline}

We compare \texttt{FLARE}  and \texttt{FLARE-NL} against the baselines evaluated by \citet{zhai2025a}. We exclude graph-related methods due to their poor performance on non-trivial transformations~\citep{zhai2025a}. 
To ensure comparability, every LLM-based method receives the same list of all explicit and implicit assumptions from the problem statement.

\begin{itemize}[topsep=2pt,parsep=0pt,leftmargin=2em]
    \item \textbf{Execution \cite{ahmaditeshnizi2024}.} Solves both formulations on a fixed problem instance and compares their optimal objective values.
    \item \textbf{EquivaMap \cite{zhai2025a}.} We re-implement EquivaMap to allow for mappings over non-scalar variables. The original mapping prompt is unmodified. We use Opus 5 as the mapping LLM.
\end{itemize}

\paragraph*{Performance.}
Table~\ref{tab:results} summarizes performance on the 54-pair NP-hard subset of FormulationBench. Our methods outperform both existing baselines. \texttt{FLARE} achieves \textbf{100\%} accuracy and is the only method that generates machine-checkable reformulation certificates. \texttt{FLARE-NL} also achieves \textbf{100\%} accuracy. Although it produces no certificate, \texttt{FLARE-NL} is \textbf{30x} faster and \textbf{25x} cheaper than \texttt{FLARE}. Both execution and EquivaMap's errors reflect limitations of instance-level validation; see Table~\ref{tab:errors} for a systematic breakdown.

\input{tables/main_experiment}

\paragraph*{Failure Modes of Instance-Level Validation.}
Table~\ref{tab:errors} summarizes errors by transformation type. Existing instance-level methods fail to catch formulation-level modeling errors, such as transformations (1) and (2), where a transformation can appear valid on the tested instance while failing as a general reformulation. EquivaMap's solution-mapping approach handles transformations (3) and (4), showing its advantage over the execution heuristic, but fails on category (5), where valid reformulations require nonlinear solution mappings. In contrast, \texttt{FLARE}'s formulation-level guarantees eliminate these false positives.

\paragraph*{Limits of ATP.}
For a valid pair of formulations, failure to produce a Lean proof results in a false negative. We do not observe this failure mode in our main experiments with Opus 5. However, we run additional experiments with GPT 5.6 Sol and DeepSeek V4 Pro (see Table \ref{tab:flare}). Both models produce false negatives attributed to ATP failure. The limits of ATP are especially prominent in DeepSeek with 79.6\% accuracy. Although Opus 5 achieves 100\% accuracy, the cost scales with the proof difficulty due to ATP. The cost standard deviation is \$1.22 and the max is \$8.02. Many problems rely on standard combinatorial results that are unavailable in Lean's standard libraries (e.g., flow decomposition). As Lean libraries improve (e.g., CSLib), \texttt{FLARE} can reduce cost by simply invoking such results rather than reproving them.

\input{tables/error_analysis}

\subsection{\texttt{FLARE-NL} Ablation Study}
\label{sec:ablation}
In Table \ref{tab:ablation}, we evaluate ablations of the \texttt{FLARE-NL} prompt to measure the importance of the reformulation definition and assumption handling across three LLMs. The \emph{Baseline} uses the full prompt; \emph{No Definition} omits the definition of reformulation; and \emph{Allow Implicit} removes implicit assumptions from the prompt and permits the model to introduce reasonable assumptions. See Appendix~\ref{sec:prompts} for prompt variants and Appendix~\ref{sec:experimental-details} for additional details.

\paragraph*{Explicit Definitions and Assumptions.} The full \texttt{FLARE-NL} prompt performs the best across all model families, indicating that reformulation verification relies on explicit definitions and assumptions. Removing the reformulation definition consistently reduces accuracy, and allowing the model to reason about implicit assumptions causes the largest degradation for every model family. Verifying reformulations requires explicit criteria and assumptions, rather than model-inferred assumptions.

\paragraph*{Frontier Model Reasoning.} \texttt{FLARE-NL} is most effective when paired with strong frontier reasoning models (also see Appendix \ref{sec:addtional-results}). Under the full prompt, Opus~5, GPT-5.6 Sol, and DeepSeek V4 Pro all achieve high accuracy, with Opus~5 correct on every pair of every run.

\input{tables/ablation}

%% file: tables/main_experiment.tex
\begin{table}[h]
\centering
\captionsetup{font=footnotesize, labelfont=bf}
\caption{Accuracy of automated reformulation checking methods on the FormulationBench dataset. \textbf{Bold} indicates the best method on a metric and \underline{underlined} indicates the second best. Metric cells report mean $\pm$ std. dev. across 3 runs; TP/FP/TN/FN reports totals. All LLM-based methods use Opus 5 with reasoning and medium effort level. Due to the high cost, we only do a single run of \texttt{FLARE}. \texttt{FLARE} and \texttt{FLARE-NL} outperform existing methods and \texttt{FLARE} is the only method that generates a machine-checkable certificate. $^{\dagger}$79.7\% of \texttt{FLARE}'s wall time is LLM inference; the remaining 20.3\% is file IO and querying the Lean LSP.}
\label{tab:results}
\scriptsize
\setlength{\tabcolsep}{4pt}
\begin{tabular}{l@{\hskip 6pt}l@{\hskip 6pt}c@{\hskip 6pt}crrrrrrrrr}
\toprule
\textbf{Method} & \textbf{Model} & \textbf{Certificate} & \textbf{Runs}& \textbf{TP} & \textbf{FP} & \textbf{TN} & \textbf{FN} & \textbf{Precision} & \textbf{Recall} & \textbf{Accuracy} & \textbf{Avg. Time} & \textbf{Avg. Cost} \\
\midrule
Execution \cite{ahmaditeshnizi2024}             & ---       & \ding{55} & 3 & 120 & 21 & 15 &   6 & \cellcolor{light-blue!20}85.1\%$\pm$0.0\% & \cellcolor{light-blue!40}\underline{95.2\%}$\pm$0.0\% & \cellcolor{light-blue!17}\underline{83.3\%}$\pm$0.0\% &   2.1s & --- \\
EquivaMap \cite{zhai2025a}             & Opus 5  & \ding{55} & 3 & 111 & 15 & 21 &  15 & \cellcolor{light-blue!26}\underline{88.1\%}$\pm$0.0\% & \cellcolor{light-blue!26}88.1\%$\pm$0.0\% & \cellcolor{light-blue!13}81.5\%$\pm$0.0\% &   6.1s & \$0.026 \\
\texttt{FLARE}             & Opus 5          & \ding{51} & 1 &  42 &  0 & 12 &   0 & \cellcolor{light-blue!50}\textbf{100.0\%} & \cellcolor{light-blue!50}\textbf{100.0\%} & \cellcolor{light-blue!50}\textbf{100.0\%} &  $^{\dagger}$410.6s & \$1.180 \\
\texttt{FLARE-NL}         & Opus 5  & \ding{55} & 3 & 126 &  0 & 36 &   0 & \cellcolor{light-blue!50}\textbf{100.0\%}$\pm$0.0\% & \cellcolor{light-blue!50}\textbf{100.0\%}$\pm$0.0\% & \cellcolor{light-blue!50}\textbf{100.0\%}$\pm$0.0\% &  13.9s & \$0.048 \\
\bottomrule
\end{tabular}
\end{table}

%% file: tables/error_analysis.tex
\begin{table}[h]
\scriptsize
\centering
\captionsetup{font=footnotesize, labelfont=bf}
\caption{Accuracy of automated reformulation checking methods segmented by challenging transformation types in the FormulationBench dataset. The \emph{Valid} column indicates if the listed transformation results in a valid reformulation. The \emph{Pairs} column gives the number of formulation pairs in each category. Worst Case reports the minimum accuracy over the transformation categories listed above. \texttt{FLARE} and \texttt{FLARE-NL} are the only methods with \textbf{100\%} accuracy in the worst case.}
\label{tab:errors}
\begin{tabular}{lcccccc}
\toprule
\textbf{Transformation} & \textbf{Valid} & \textbf{Pairs} & \textbf{Execution \cite{ahmaditeshnizi2024}} & \textbf{EquivaMap \cite{zhai2025a}} & \textbf{\texttt{FLARE}} & \textbf{\texttt{FLARE-NL} } \\
\midrule
\textbf{1.} Base-10 Representation   & \ding{55} & 2 & 0\%   & 0\%   & \textbf{100\%} & \textbf{100\%} \\
\textbf{2.} Addition of Invalid Cutting Planes & \ding{55} & 3 & 0\%   & 0\%   & \textbf{100\%} & \textbf{100\%} \\
\textbf{3.} Rescaled Objective & \ding{51} & 2 & 0\%   & \textbf{100\%} & \textbf{100\%} & \textbf{100\%} \\
\textbf{4.} Different Formulation (Same Objective) & \ding{55} & 2 & 0\%   & \textbf{100\%} & \textbf{100\%} & \textbf{100\%} \\
\textbf{5.} Non-Linear Solution Maps & \ding{51} & 5 & \textbf{100\%} & 0\%   & \textbf{100\%}  & \textbf{100\%} \\
\midrule
\rowcolor{black!10}
\textbf{Worst Case} & & & 0\% & 0\% & \textbf{100\%} & \textbf{100\%} \\
\bottomrule
\end{tabular}
\end{table}

%% file: tables/ablation.tex
\begin{table}[ht]
\centering
\captionsetup{font=footnotesize, labelfont=bf}
\caption{Ablation study of \texttt{FLARE-NL} across LLM model and prompt variations. \textit{No Definition} omits the reformulation definition from the prompt and \textit{Allow Implicit} removes implicit assumptions from the prompt and permits the model to introduce reasonable assumptions. Metric cells report mean $\pm$ std. dev. across 3 runs; TP/FP/TN/FN reports totals. All models have reasoning enabled and the reasoning effort level is shown following the model name.}
\label{tab:ablation}
\scriptsize
\setlength{\tabcolsep}{4pt}
\begin{tabular}{llrrrrrrrrr}
\toprule
\textbf{Model (Effort)} & \textbf{Variant} & \textbf{TP} & \textbf{FP} & \textbf{TN} & \textbf{FN} & \textbf{Precision} & \textbf{Recall} & \textbf{Accuracy} & \textbf{Avg Time} & \textbf{Avg Cost} \\
\midrule
\multirow{3}{*}{\makecell[l]{Opus 5\\\textit{(medium)}}}
  & Baseline       & 126 & 0 & 36 &  0 & \cellcolor{light-blue!50}\textbf{100.0$\pm$0.0\%} & \cellcolor{light-blue!50}\textbf{100.0$\pm$0.0\%} & \cellcolor{light-blue!50}\textbf{100.0$\pm$0.0\%} & 13.9s & \$0.048 \\
  & No Definition  & 120 & 0 & 36 &  6 & \cellcolor{light-blue!50}\textbf{100.0$\pm$0.0\%} & \cellcolor{light-blue!42}95.2$\pm$0.0\% & \cellcolor{light-blue!44}96.3$\pm$0.0\% & 12.0s & \$0.039 \\
  & Allow Implicit & 121 & 6 & 30 &  5 & \cellcolor{light-blue!42}95.3$\pm$0.1\% & \cellcolor{light-blue!43}96.0$\pm$1.4\% & \cellcolor{light-blue!39}93.2$\pm$1.1\% & 18.4s & \$0.052 \\
\midrule
\multirow{3}{*}{\makecell[l]{GPT-5.6 Sol\\\textit{(medium)}}}
  & Baseline       & 120 & 0 & 36 &  6 & \cellcolor{light-blue!50}\textbf{100.0$\pm$0.0\%} & \cellcolor{light-blue!42}95.2$\pm$0.0\% & \cellcolor{light-blue!44}96.3$\pm$0.0\% & 14.6s & \$0.034 \\
  & No Definition  & 120 & 0 & 36 &  6 & \cellcolor{light-blue!50}\textbf{100.0$\pm$0.0\%} & \cellcolor{light-blue!42}95.2$\pm$0.0\% & \cellcolor{light-blue!44}96.3$\pm$0.0\% & 12.3s & \$0.027 \\
  & Allow Implicit &  93 & 6 & 30 & 33 & \cellcolor{light-blue!40}93.9$\pm$0.0\% & \cellcolor{light-blue!6}73.8$\pm$0.0\% & \cellcolor{light-blue!10}75.9$\pm$0.0\% & 20.7s & \$0.040 \\
\midrule
\multirow{3}{*}{\makecell[l]{DeepSeek V4 Pro\\\textit{(high)}}}
  & Baseline       & 118 & 0 & 36 &  8 & \cellcolor{light-blue!50}\textbf{100.0$\pm$0.0\%} & \cellcolor{light-blue!40}93.7$\pm$1.4\% & \cellcolor{light-blue!42}95.1$\pm$1.1\% & 60.1s & \$0.005 \\
  & No Definition  & 113 & 0 & 36 & 13 & \cellcolor{light-blue!50}\textbf{100.0$\pm$0.0\%} & \cellcolor{light-blue!33}89.7$\pm$1.4\% & \cellcolor{light-blue!37}92.0$\pm$1.1\% & 56.9s & \$0.004 \\
  & Allow Implicit & 100 & 5 & 31 & 26 & \cellcolor{light-blue!42}95.3$\pm$1.5\% & \cellcolor{light-blue!16}79.4$\pm$3.6\% & \cellcolor{light-blue!18}80.9$\pm$2.1\% & 71.2s & \$0.005 \\
\bottomrule
\end{tabular}
\end{table}

%% file: sections/7_conclusion.tex
\section{Conclusion}
\label{sec:conclusion}

LLMs are increasingly used for optimization modeling but lack formal guarantees, motivating methods to automatically verify LLM-generated formulations. We develop \texttt{FLARE}, the first automated approach for producing machine-checkable reformulation certificates. \texttt{FLARE}  combines a constructive definition of reformulation and ATP to verify reformulation claims at the formulation level, while existing methods only operate at the instance level. Furthermore, we introduce \texttt{FLARE-NL}, an LLM proxy that trades formal guarantees for improved speed and cost. Both methods outperform existing baselines and achieve \textbf{100\%} accuracy on the NP-hard subset of FormulationBench.

\paragraph*{Limitations and Future Work.}
Several limitations motivate future work on formally verified AI-driven optimization modeling.

\begin{itemize}[topsep=2pt,parsep=0pt,leftmargin=2em]

\item \textbf{Deterministic formalization}. The \texttt{FLARE} reformulation certificate requires faithful formalization. Deterministically translating accepted MILP modeling standards into Lean eliminates formalization as a source of error.

\item \textbf{Certificates of invalidity.} \texttt{FLARE} currently searches only for certificates of validity. Future work could develop methods for proving that no reformulation construction exists under a fixed parameter map. For current ATP systems, proving this type of non-existence guarantee will likely be substantially more difficult than proving existence by exhibiting an explicit construction.

\item \textbf{Alternative formulation relations.} Our constructive definition requires forward and backward maps over the entire feasible region. Thus, it may reject a reformulation that removes non-optimal feasible points while preserving the optimal feasible points. Future work could formalize weaker, optimality-preserving notions of reformulation. It could also formalize formulation strength as a well-defined relation (e.g., by proving containment between the
projected linear relaxations of two formulations).

\item \textbf{Real-world formulations.} The formulations of FormulationBench are drawn from academic benchmarks. Future work should extend FormulationBench to demonstrate that \texttt{FLARE} and ATP can scale to larger, more complex industrial MILP formulations.

\item \textbf{Verification for automated modeling}. Our results raise questions about how to integrate verification into LLM-driven modeling and automated algorithm design pipelines (e.g., EvoCut \cite{yazdani2025}). A natural design is to use \texttt{FLARE-NL} for rapid screening and \texttt{FLARE} for final certification.

\end{itemize}

%% file: sections/acknowledgements.tex
\section*{Acknowledgements}

We thank the anonymous referees for thoughtful feedback that significantly improved the clarity of the paper.
We also thank Nichie Supatgiat for her contributions to FormulationBench and helpful discussions.
We thank Refine.ink for detailed manuscript feedback.
MU gratefully acknowledges support from
the Office of Naval Research under award N000142412306,
Air Force Office of Scientific Research under award FA9550-26-1-0012,
the Alfred P. Sloan Foundation,
the Stanford Institute for Human-Centered Artificial Intelligence (HAI),
and from IBM Research as a founding member of Stanford Institute for Human-centered Artificial Intelligence.
EV gratefully acknowledges support from the National Science Foundation under award CCF-2338226 and the AI2050 program at Schmidt Sciences.
Any opinions, findings, and conclusions or recommendations expressed in this material are those of the author(s) and do not necessarily reflect the views of these funders.

%% file: appendix/1_constructive_audet_proof.tex
\section{Proof of Proposition \ref{prop:constructive-audet}}
\label{sec:proofs}

\subimport{proofs/}{constructive_audet}

%% file: proofs/constructive_audet.tex
\constructiveaudet*

\begin{proof}
Let ${\mathcal{M} = (\mathcal{P}, \mathcal{F}, f_0)}$ and ${\mathcal{M}' = (\mathcal{P}', \mathcal{F}', f'_0)}$ be formulations. Assume $\mathcal{M}'$ is a constructive reformulation of $\mathcal{M}$. By definition, there exists a reformulation construction $\Phi(\mathcal{M},\mathcal{M}') = (\Phi_{\mathrm{p}}, \Phi_{\text{fwd}}, \Phi_{\text{bwd}}, \Phi_{\text{obj}})$ such that the conditions of Definition \ref{def:constructive} hold. Fix an instance $p \in {\cal P}$ and let $p' = \Phi_p(p)$.

We first show that if $\mathcal{M}(p)$ has an optimal solution, then $\mathcal{M}'(p')$ also has an optimal solution. Let $x^* \in {\cal F}(p)$ be optimal for ${\cal M}(p)$. By forward feasibility, $x' = \Phi_{\text{fwd}}(x^*;p) \in {\cal F}'(p')$. We now claim that $x'$ is optimal for $\mathcal{M}'(p')$. Suppose this was not true. There must exist an $\hat{x}' \in {\cal F}'(p')$ such that $f_0'(\hat{x}') < f_0'(x')$. By backward feasibility, $\hat{x} = \Phi_{\text{bwd}}(\hat{x}';p) \in {\cal F}(p)$. By objective preservation:
$$
f'_0(\hat{x}';p') = \Phi_{\text{obj}}(f_0(\hat{x};p))
\quad\text{and}\quad
f'_0(x';p') =\Phi_{\text{obj}}(f_0(x^*;p)).
$$
Thus,
$$
\Phi_{\text{obj}}(f_0(\hat{x};p)) <
\Phi_{\text{obj}}(f_0(x^*;p)).
$$
Since $\Phi_{\text{obj}}$ is strictly increasing, this implies $f_0(\hat{x};p) < f_0(x^*;p)$ which contradicts the optimality of $x^*$. Using an identical argument by contradiction, we can show that every optimal solution of $\mathcal{M}'(p')$ is mapped to an optimal solution of ${\mathcal{M}}(p)$ by the backward mapping $\Phi_{\text{bwd}}(\cdot;p)$. Definition~\ref{def:constructive} requires $\Phi_{\text{bwd}}(\cdot;p)$ to be computable in polynomial time. Therefore, given any optimal solution of $\mathcal{M}'(p')$, an optimal solution of $\mathcal{M}(p)$ can be recovered in polynomial time. Hence $\mathcal{M}'$ is an Audet reformulation of $\mathcal{M}$.

For the instantiated claim, fix $p$ and $p'=\Phi_p(p)$. Definition~\ref{def:quasi-karp-equivalence} requires a mapping-producing algorithm $\mathcal{A}$. Let $\mathcal{A}(\mathcal{M}(p),\mathcal{M}'(p'))$ output the backward mapping $\Phi_{\mathrm{bwd}}$ with $p$ bound as a constant, representing the mapping
$$
f(x') = \Phi_{\mathrm{bwd}}(x';p).
$$
This only copies $p$ into a fixed-size algorithm, so $\mathcal{A}$ runs in polynomial time for all $p \in \mathcal{P}$ and $p' \in \mathcal{P}'$. By Definition~\ref{def:constructive}, $f$ is polynomial-time computable, and the argument above shows $f$ maps every optimal solution of $\mathcal{M}'(p')$ to an optimal solution of $\mathcal{M}(p)$. Thus $\mathcal{M}'(p')$ is Quasi-Karp equivalent to $\mathcal{M}(p)$.
\end{proof}

%% file: appendix/2_lean_formalization.tex
\section{Lean Formalization}
\label{sec:lean-formalization}

We now formalize the \emph{formulation} (Definition \ref{def:formulation}) and \emph{reformulation} (Definition \ref{def:constructive}) definitions in Lean. For both definitions, we utilize a Lean \texttt{structure}. A \texttt{structure} is a collection of named fields and their types. See Listing \ref{lst:lean-formalizations} for the \texttt{MILPFormulation} and \texttt{MILPReformulation} structures. In the following sections, we provide details on how each definition is formalized.

\subsection{Formulation}

A formulation ${\mathcal{M} = (\mathcal{P}, \mathcal{F}, f_0)}$ is represented by the \texttt{MILPFormulation} structure. The \texttt{Params} field encodes the parameter space $\mathcal{P}$. Next, the \texttt{Vars} and \texttt{feasible} fields together encode the feasible region $\mathcal{F}$. Lastly, \texttt{obj} encodes the objective function $f_0$. Notice that both \texttt{feasible} and \texttt{obj} are functions of \texttt{Params} and \texttt{Vars}. See Listing \ref{lst:tsp-lean-formalization} for an example Lean formalization of the TSP \ref{eq:mtz} formulation.

\paragraph*{Discrepancies.}
This definition does not restrict \texttt{Vars} to $\R^n$ nor does it assert linearity conditions on either \texttt{feasible} or \texttt{obj}. This additional typing would add modeling complexity to \texttt{Vars}, reduce legibility, and place unnecessary burden on ATP. Our primary motivation is not to prove if a MILP formulation is well-formed. Hence, we make the practical decision to exclude this typing.

\subsection{Reformulation}

We formalize reformulation with the \texttt{MILPReformulation} structure. This defines the reformulation construction $\Phi(\mathcal{M},\mathcal{M}') = (\Phi_{\mathrm{p}}, \Phi_{\text{fwd}}, \Phi_{\text{bwd}}, \Phi_{\text{obj}})$ along with the four conditions specified in Definition \ref{def:constructive}. The fields \texttt{paramMap}, \texttt{fwd}, \texttt{bwd}, and \texttt{obj} encode $\Phi_{\mathrm{p}}$, $\Phi_{\text{fwd}}$, $\Phi_{\text{bwd}}$, and $\Phi_{\text{obj}}$ respectively. The fields \texttt{fwd\_feas} and \texttt{bwd\_feas} encode the forward and backward feasibility conditions. The two objective mapping conditions are encoded by \texttt{fwd\_obj} and \texttt{bwd\_obj}. Lastly, \texttt{objMap\_mono} encodes the strict monotonicity objective condition. A formulation \texttt{G} is proven to be a reformulation of \texttt{F} by declaring a definition of type \texttt{MILPReformulation F G}. This structure encodes both the witnessing reformulation construction and proves it obeys the necessary conditions. Proving \texttt{G} is \textit{not} a reformulation of \texttt{F} requires proving the type \texttt{MILPReformulation F G} is uninhabited (i.e., no construction satisfying the conditions exists).

\paragraph*{Discrepancies.}
The Lean formalization does not enforce that the backward mapping $\Phi_{\text{bwd}}$ is computable in polynomial time. Formalizing computational complexity in Lean is non-trivial and adds substantial proof burden to ATP. In our experiments, we observe no examples of LLMs producing non-polynomial-time backward maps, so we om it the condition from the Lean proof obligation.

\begin{lstlisting}[
    style=lean,
    float=h,
    basicstyle=\scriptsize\ttfamily,
    caption={Lean formalizations: \texttt{MILPFormulation} and \texttt{MILPReformulation}.},
    label={lst:lean-formalizations}
]
import Mathlib.Tactic
import Mathlib.Data.Real.Basic
import Mathlib.Order.Basic

structure MILPFormulation where
  Params   : Type                                        -- Parameter space
  Vars     : Params → Type                              -- Variables
  feasible : (p : Params) → Vars p → Prop              -- Feasible region
  obj      : (p : Params) → Vars p → ℝ                 -- Objective function

structure MILPReformulation (F G : MILPFormulation) where
  paramMap    : F.Params → G.Params                    -- Parameter mapping
  fwd         : (p : F.Params) → F.Vars p →           -- Forward mapping
                  G.Vars (paramMap p)
  bwd         : (p : F.Params) → G.Vars (paramMap p) → -- Backward mapping
                  F.Vars p
  fwd_feas    : ∀ p x, F.feasible p x →                -- Forward feasibility condition
                  G.feasible (paramMap p) (fwd p x)
  bwd_feas    : ∀ p x', G.feasible (paramMap p) x' →   -- Backward feasibility condition
                  F.feasible p (bwd p x')
  objMap      : ℝ → ℝ                                  -- Objective mapping
  objMap_mono : StrictMono objMap                      -- Objective monotonicity
  fwd_obj     : ∀ p x, F.feasible p x →                -- Forward objective condition
                  G.obj (paramMap p) (fwd p x) = objMap (F.obj p x)
  bwd_obj     : ∀ p x', G.feasible (paramMap p) x' →   -- Backward objective condition
                  G.obj (paramMap p) x' = objMap (F.obj p (bwd p x'))
\end{lstlisting}

\begin{lstlisting}[
    style=lean,
    float=h,
    basicstyle=\scriptsize\ttfamily,
    caption={An example Lean formalization of the TSP \ref{eq:mtz} formulation.},
    label={lst:tsp-lean-formalization}
]
structure Params where
  n : ℕ                           -- number of cities
  c : Fin n → Fin n → ℝ         -- arc cost
  hn : 2 ≤ n

structure Vars (p : Params) where
  x : Fin p.n → Fin p.n → ℤ     -- arc indicator
  u : Fin p.n → ℝ                -- position

structure Feasible (p : Params) (v : Vars p) : Prop where
  -- Each city has exactly one outgoing arc
  hout : ∀ i : Fin p.n, ∑ j : Fin p.n, v.x i j = 1
  -- Each city has exactly one incoming arc
  hin : ∀ j : Fin p.n, ∑ i : Fin p.n, v.x i j = 1
  -- MTZ subtour elimination
  hmtz : ∀ (i : Fin p.n) (j : Fin p.n), i.val ≠ 0 → j.val ≠ 0 → i ≠ j →
    v.u i - v.u j + (p.n : ℝ) * (v.x i j : ℝ) ≤ (p.n : ℝ) - 1
  -- Depot position fixed to 1
  hu_depot : haveI : NeZero p.n := ⟨by have := p.hn; omega⟩; v.u 0 = 1
  hx_bin : ∀ (i j : Fin p.n), v.x i j = 0 ∨ v.x i j = 1
  -- u ∈ [2, n] for non-depot cities
  hu_lo : ∀ i : Fin p.n, i.val ≠ 0 → 2 ≤ v.u i
  hu_hi : ∀ i : Fin p.n, v.u i ≤ (p.n : ℝ)
  -- No self-loops
  hx_no_self : ∀ i : Fin p.n, v.x i i = 0

-- Minimize total arc cost
def obj (p : Params) (v : Vars p) : ℝ :=
  ∑ i : Fin p.n, ∑ j : Fin p.n, p.c i j * (v.x i j : ℝ)
\end{lstlisting}

%% file: appendix/3_implementation.tex
\section{\texttt{FLARE}  Implementation Details}
\label{sec:implementation}

This section outlines the implementation details for every component of \texttt{FLARE} (Figure \ref{fig:flare}).\footnote{\texttt{FLARE} is implemented in the \texttt{milp-flare} Python package; see \url{https://flare.henryrobbins.com}.} We use a general-purpose coding agent harness (Claude Code, Codex, and OpenCode) for both autoformalizing MILP formulations and ATP. This design decision was inspired by recent work achieving competitive ATP performance with Claude Code~\cite{liu2026e}. In Section \ref{sec:agent-harness}, we describe the agent harness prompt, working directory, and sandbox. To enable the agent to effectively compose Lean proofs, we use the Lean-LSP-MCP (Section \ref{sec:lean-lsp-mcp}) and two custom agent skills (Section \ref{sec:agent-skills}). After the agent exits, \texttt{FLARE}  does a final verification to check if a reformulation certificate was successfully generated (Section \ref{sec:flare-verification}).

\subsection{Agent Harness}
\label{sec:agent-harness}

We programmatically initialize the agent harness in a headless mode using the CLI. The agent is given instructions and provided with a working directory containing the necessary context. To isolate the agent's working directory and avoid duplicating the Lean environment (the Mathlib library is over 5GB), we run \texttt{FLARE} inside a Docker container.

\paragraph*{Agent Prompt.}
In the agent prompt (Listing \ref{lst:flare-agent}), we provide the agent with instructions to (1) formalize both MILP formulations as \texttt{MILPFormulation} structures \texttt{A} and \texttt{B} and (2) attempt to declare a definition of type \texttt{MILPReformulation A B} where \texttt{paramMap} formalizes the fixed parameter map.

\paragraph*{Working Directory Files.}
The working directory is initialized with the following context:
\begin{itemize}[topsep=2pt,parsep=0pt,leftmargin=2em]
    \item templated \LaTeX descriptions of both MILP formulations (Listing \ref{lst:formulation-template}) and the parameter map (Listing \ref{lst:parameter-map-template}),
    \item Gurobi Python implementations for both MILP formulations and the parameter mapping, and
    \item a Lean file \texttt{Common.lean} with \texttt{MILPFormulation} and \texttt{MILPReformulation} definitions.
\end{itemize}

Both the templated \LaTeX{} descriptions and Gurobi Python implementations are populated by the FormulationBench JSON files. The \texttt{formulation-bench} Python package provides utilities to construct Gurobi Python implementations from the JSON file code snippets. See Appendix \ref{sec:formulation-bench-details} for additional details on the FormulationBench dataset structure.

\paragraph*{Docker.}

We pre-build a Docker image called \texttt{flare-agent} with every agent CLI (Claude Code, Codex, and OpenCode), the Lean-LSP-MCP, and a Lake-built Lean environment with Mathlib pre-compiled. When \texttt{FLARE} is invoked, it creates a working directory on the host with all the necessary files. It then creates a Docker container from this image and copies the working directory into the container. To use the pre-built Lean environment, we symlink \texttt{.lake} from the pre-built location into the agent's working directory. This configuration prevents duplication of the Lean environment while ensuring \texttt{FLARE} can run in parallel without agents impacting each other's Lean environment. Previously, we attempted to isolate agents with the permissions and sandboxing mechanisms provided by each agent harness. We found the implementations to be immature; Docker was the only reliable way to ensure the agent couldn't access files outside its working directory.

\subsection{Lean-LSP-MCP}
\label{sec:lean-lsp-mcp}

Lean-LSP-MCP~\cite{dressler2025}\footnote{\url{https://github.com/oOo0oOo/lean-lsp-mcp}} is a Model Context Protocol (MCP) server providing MCP Tools for Lean theorem proving. It enables the agent to efficiently generate, debug, and compile Lean proofs and has been utilized by numerous Lean ATP methods~\cite{breen2025,liu2026e,ju2026a}. The Lean Language Server Protocol (LSP) traditionally allows editors like VS Code to get rich, interactive feedback from a running Lean process. The Lean-LSP-MCP allows the agent to communicate directly with this LSP and obtain feedback like a human theorem prover. It offers a broad range of tools including:

\begin{itemize}[topsep=2pt,parsep=0pt,leftmargin=2em]
    \item \texttt{lean\_goal}. Get the proof goal at a specific location in the file. This allows the agent to observe precisely what sub-goal must be proved to continue making progress.
    \item \texttt{lean\_diagnostic\_messages}. Retrieve all the diagnostic messages for a Lean file. This allows the agent to verify if the file is compiling and, if not, where fixes are necessary.
    \item \texttt{lean\_verify}. Verify the soundness of a proof. There are two conditions where a proof compiles, but is actually unsound: (1) a new \texttt{axiom} was added or (2) the \texttt{sorry} tactic is used to complete a goal (see Appendix \ref{sec:flare-verification} for a further discussion). This tool allows the agent to verify that neither condition holds.
    \item \texttt{lean\_multi\_attempt}. Attempt multiple tactics at a proof position and return the new goal state for each. This allows the agent to explore strategies for making progress in parallel.
    \item \texttt{lean\_hover\_info}. Retrieve documentation for symbols, terms, and expressions. This allows the agent to view underlying definitions and reduces hallucination.
\end{itemize}

\subsection{Agent Skills}
\label{sec:agent-skills}

Agent skills are an ``open format for extending AI agent capabilities with specialized knowledge and workflows.'' They are simply a directory containing context, scripts, or other resources related to a specific task. The directory contains a \texttt{SKILL.md} file which tells the agent the skill's name and provides instructions on when to use it. Skills allow for \emph{progressive disclosure} where the agent loads additional task-specific context as needed.

We define two custom agent skills, \texttt{lean-milp-formulation} and \texttt{lean-milp-reformulation}, for autoformalizing MILP formulations and proving reformulations respectively. The agent prompt explicitly instructs the agent to invoke these skills. Both skills contain instructions along with a template Lean file containing detailed comments about Lean modeling choices and file structure. In addition to our custom skills, we use the general-purpose Lean 4 skills \cite{freer2025}.\footnote{\url{https://github.com/cameronfreer/lean4-skills}}

\subsection{Final Verification}
\label{sec:flare-verification}

After the agent exits, we inspect the working directory. First, we verify the presence of \texttt{A/Formulation.lean}, \texttt{B/Formulation.lean}, and \texttt{Reformulation.lean}. If all of these files are present and compile, we then check that \texttt{Reformulation.lean} contains a \texttt{MILPReformulation A B} definition. Finally, we guard against two conditions where a proof compiles but is unsound: use of the \texttt{sorry} tactic and the introduction of non-standard axioms. The \texttt{sorry} tactic can skip any proof obligation and indicates the proof has yet to be formally verified. We emit \texttt{\#print axioms} for the \texttt{MILPReformulation A B} definition, which reports its transitive axiom dependencies; a \texttt{sorry} surfaces here as \texttt{sorryAx}. We require that the proof depend only on Lean's three standard axioms (\texttt{propext}, \texttt{Classical.choice}, and \texttt{Quot.sound}). If all files compile and the \texttt{MILPReformulation A B} definition is \texttt{sorry}-free and introduces no new axioms, formulation \texttt{B} is deemed a reformulation of \texttt{A} under the fixed parameter mapping.

%% file: appendix/4_prompts.tex
\section{Prompts}
\label{sec:prompts}

\begin{lstlisting}[
    style=jinja,
    basicstyle=\scriptsize\ttfamily,
    caption={\texttt{FLARE}  Agent Prompt},
    label={lst:flare-agent}
]
You are a mixed-integer linear programming (MILP) and Lean 4 expert. You have been tasked with formalizing two MILP formulations in Lean 4 and then proving that formulation B is a reformulation of formulation A.

## Working Directory Structure

The working directory will be initialized with the following structure:

```
  +-- A
  |   +-- Formulation.lean    # Write the Lean 4 formalization of Formulation A here
  |   +-- formulation.md      # Natural language description of Formulation A
  |   \-- solve.py            # Python script that solves Formulation A
  +-- B
  |   +-- Formulation.lean    # Write the Lean 4 formalization of Formulation B here
  |   +-- formulation.md      # Natural language description of Formulation B
  |   \-- solve.py            # Python script that solves Formulation B
  +-- Common.lean             # Definition of `MILPFormulation` and `MILPReformulation`
  +-- Reformulation.lean      # Write the reformulation proof here
  +-- map.md                  # The parameter mapping from Formulation A to Formulation B
  +-- map.py                  # Python script computing B's parameters from A's
  +-- lake-manifest.json      # DO NOT EDIT!
  +-- lakefile.toml           # DO NOT EDIT!
  \-- lean-toolchain          # DO NOT EDIT!
```

## Parameter Map
The parameter mapping of the reformulation construction is *given*, not something you search for. `map.md` states, for each parameter of Formulation B, how it is computed from the parameters of Formulation A; `map.py` is the same mapping as an executable Python script. Read `map.md` before writing either `Formulation.lean` file.

## Workflow
1. Read `map.md` (and `map.py` where the LaTeX is ambiguous) to understand the given parameter mapping.
2. Utilize the `lean-milp-formulation` skill to generate both `Formulation.lean` files.
3. Validate each `Formulation.lean` file with `mcp__lean-lsp__lean_diagnostic_messages`. Fix any issues that arise and repeat until both files compile cleanly.
4. Run `Bash(lake build A.Formulation B.Formulation)` to materialize their oleans.
5. Determine whether formulation B is a mathematical reformulation of formulation A *under the given parameter mapping*.
6. If B is a reformulation of A, use the `lean-milp-reformulation` skill to generate a compiled `MILPReformulation` proof in `Reformulation.lean`, with its `paramMap` field implementing the mapping in `map.md`.
7. Validate `Reformulation.lean` with `mcp__lean-lsp__lean_verify` to ensure there are no stubs. Fix any issues that arise and repeat until both files compile cleanly.

## Rules
- IMPORTANT: Only read/write files that exist in *this* working directory. Do not navigate outside of it for any reason.
- The `paramMap` field of your `MILPReformulation` MUST implement the mapping given in `map.md`. Do not substitute a different parameter mapping, even if another one would make the proof easier. If B is not a reformulation of A under *this* mapping, that is a negative result -- report it as described below rather than switching mappings.
- `map.md` also constrains how you formalize each `Formulation.lean`: B's `Params` structure must have exactly the parameters that `map.md` computes, and A's `Params` must have the parameters `map.md` computes them from. If you cannot express the mapping as a total function `A.Params -> B.Params`, the mismatch is in your formalization -- fix the `Params` structures rather than adapting the mapping.
- DO NOT EDIT `map.md` or `map.py`.
- The Lean project root is the current directory. Use `import A.Formulation` and `import B.Formulation`.
- You MUST use the lean-lsp MCP tools (mcp__lean-lsp__*) to check compilation as you work. Before doing anything else, verify the server is available by calling `mcp__lean-lsp__lean_diagnostic_messages` on `Common.lean`. If the first probe reports the tool as unregistered or still connecting, that is expected -- probe again until it answers, up to 3 attempts. If the *same* server-level failure (e.g. "Failed to start Lean language server", tool not registered) persists, conclude the environment is unusable: write `MCP_UNAVAILABLE: <error>` to Reformulation.lean and exit. Do not fall back to `lake env lean` to perform compilation.
- If `lake build` fails, treat the error message as a real signal -- read it, fix the cause (typically a typo, missing import, or stale olean), and retry. Only after the *same* failure persists across two clean retries should you write `LAKE_BUILD_FAILED: <error>` to Reformulation.lean and exit. Persistent toolchain-level failures (e.g. elan/toolchain missing) are the exit case; ordinary compile errors are not.
- Generate both Formulation.lean files before attempting the reformulation proof.
- Confirm the final reformulation proof with `mcp__lean-lsp__lean_verify` on the `MILPReformulation` definition. The returned axioms must NOT contain `sorryAx` -- if it does, the proof has a stub and you must finish it.
- You are expected to iterate on the proof until it compiles and `lean_verify` reports no `sorryAx`. A non-trivial reformulation proof typically runs 100+ lines with many manual `refine`/`rcases`/`simp` steps and multiple rounds of compile-error fixing. You should only exit before this point if there is concrete evidence that B is *not* a reformulation of A under the given assumptions.

## Common Mistakes
- Interpret `lean_diagnostic_messages` carefully: `success:true, items:[]` means the file compiles cleanly. `success:false, items:[]` typically means imports aren't built yet -- build them, don't assume the file is broken. Real errors come back as `items` with severity/message fields.
- If there issues with the Lean environment or MCP tools, it is imperative to handle them as described in the rules above. Report and exit.
- If you are stuck proving the reformulation due missing assumptions in the `Formulation.lean` files, first verify that the assumptions specified in `formulation.md` are all present in the corresponding `Formulation.lean`. If any assumptions are missing, add them. Otherwise, DO NOT ADD ASSUMPTIONS YOURSELF. Instead, report the missing necessary assumptions as an issue preventing the reformulation proof in `Reformulation.lean` and exit.

## Available Tools
**Filesystem:** `Bash(ls ./*)`, `Bash(find ./*)`,
**Read:** `Read(./*)`, `Bash(cat ./*)`, `Bash(head ./*)`, `Bash(tail ./*)`, `Bash(less ./*)`, `Bash(more ./*)`, `Bash(bat ./*)`
**Edit:** `Edit(./*)`
**Write:** `Write(./*)`
**Skills:** `lean4:lean4`, `lean-milp-formulation`, `lean-milp-reformulation`
**MCP Tools:** `mcp__lean-lsp__*`
**Lake:** `Bash(lake env lean:*)`, `Bash(lake build:*)`
\end{lstlisting}

\begin{lstlisting}[
    style=jinja,
    basicstyle=\scriptsize\ttfamily,
    numbers=left,
    numberstyle=\tiny\color{gray}\texttt,
    numbersep=8pt,
    xleftmargin=2.2em,
    framexleftmargin=1.6em,
    caption={\texttt{FLARE-NL} Prompt. The \textit{No Definition} variant omits L9-27 and drops the $\Phi_{p}$ symbol from L31. The \textit{Allow Implicit} variant omits L38 as well as implicit assumptions and constraints from the formulation descriptions. The formulation descriptions are populated by Listing \ref{lst:formulation-template} and the parameter mapping is populated by Listing \ref{lst:parameter-map-template}.},
    label={lst:llm-equivalence}
]
You are given the following two Mixed-Integer Linear Programming (MILP) formulations. You are tasked with deciding if formulation B is a reformulation of formulation A.

## Formulations

{{ formulation_a }}

{{ formulation_b }}

## Definitions

Use the following definition of formulation and reformulation to guide your reasoning.

**Formulation.** A MILP *formulation* $A$ is a tuple ${A = (P, F, f_0)}$ with parameter space $P$, feasible region $F(p) \subseteq R^{n(p)}$, and objective function $f_0$. For instance $p \in P$, the feasible region $F(p)$ is defined by $m(p)$ linear constraints, $f_i(\cdot ; p) : R^{n(p)} \to R$ for all $i \in [m(p)]$. The first $k(p) \leq n(p)$ variables are integers. The feasible region is
$$F(p) = \{x \in \mathbb{Z}^{k(p)} \times R^{n(p)-k(p)}~|~f_i(x;p) \leq 0 ~\forall i\in[m(p)]\}.$$
The objective is to minimize the linear function $f_0(\cdot ; p) : R^{n(p)} \to R$. A formulation $A$ is *instantiated* with an *instance* $p \in P$. We denote an instantiated formulation as ${A(p) = (F(p), f_0(p))}$.

**Reformulation.** Let $A = (P, F, f_0)$ and $B = (P', F', f'_0)$ be formulations. A *reformulation construction* from $A$ to $B$ is a tuple $\Phi(A, B) = (\Phi_{p},\, \Phi_{fwd},\, \Phi_{bwd},\, \Phi_{\text{obj}})$ consisting of:
- a parameter mapping $\Phi_{p} : P \to P'$,
- a forward mapping $\Phi_{fwd}(\cdot;p) : R^{n(p)} \to R^{n'(\Phi_{p}(p))}$,
- a backward mapping $\Phi_{bwd}(\cdot;p) : R^{n'(\Phi_{p}(p))} \to R^{n(p)}$ computable in polynomial time, and
- an objective mapping $\Phi_{\text{obj}} : R \to R$.

$B$ is a *reformulation* of $A$ if there exists a reformulation construction satisfying the following conditions for every instance $p \in P$, with $p' = \Phi_{p}(p)$:
- **Forward feasibility.** For all feasible points $x \in F(p)$, the forward mapping maps to a feasible point $x' = \Phi_{fwd}(x;p) \in F'(\Phi_{p}(p))$.
- **Backward feasibility.** For all feasible points $x' \in F'(p')$, the backward mapping maps to a feasible point $x = \Phi_{bwd}(x';p) \in F(p)$.
- **Objective mapping.** The forward and backward mappings induce an objective mapping. The following two conditions must hold. (1) For all feasible points ${x \in F(p)}$, the forward mapped point $x' = \Phi_{fwd}(x;p)$ has objective value $f_0'(x';p') = \Phi_{\text{obj}}(f_0(x;p))$. (2) For all feasible points ${x' \in F'(p')}$, the backward mapped point is $x = \Phi_{bwd}(x';p)$ and $f_0'(x';p') = \Phi_{\text{obj}}(f_0(x;p))$.
- **Strictly monotone.** The objective mapping $\Phi_{\text{obj}}$ is strictly monotonically increasing.

## Parameter Mapping

The parameter mapping $\Phi_{p}$ is *given*. It computes each parameter of formulation B from the parameters of formulation A.

{{ parameter_map }}

## Instructions

- Decide whether B is a reformulation of A under *this* parameter mapping. Do not substitute a different one, even if another mapping would make B a reformulation of A.
- Do NOT make any assumptions about the formulation or parameter space that are not explicitly stated in the formulation descriptions.
- When uncertain, state that formulation B is *not* a reformulation of A.
- Provide a short summary of your conclusion (at most 2,500 characters) and a final determination of whether B is a reformulation of A (true or false).
\end{lstlisting}

\clearpage
\begin{lstlisting}[
    style=jinja,
    basicstyle=\scriptsize\ttfamily,
    caption={Formulation Prompt Template. Rendered as \texttt{formulation.md} for \texttt{FLARE} and \texttt{\{\{ formulation\_(a|b) \}\}} in the \texttt{FLARE-NL} prompt.},
    label={lst:formulation-template}
]
# {{ problem_name }}

## Problem Description

{{ problem_description }}

## Formulation

### Parameters

{% for name, p in parameters.items() %}
- **{{ name }}** (type: {{ p.type.value }}, shape: `{{ p.shape }}`): {{ p.description }}
{% endfor %}

{% if assumptions %}

### Assumptions

{% for a in assumptions %}
- {{ a.description }}
$${{ a.formulation }}$$
{% endfor %}
{% endif %}

{% if definitions %}

### Definitions

{% for name, d in definitions.items() %}
- **{{ name }}**: {{ d.description }}
$${{ d.formulation }}$$
{% endfor %}
{% endif %}

### Variables

{% for name, v in variables.items() %}
- **{{ name }}** (type: {{ v.type.value }}, shape: `{{ v.shape }}`): {{ v.description }}
{% endfor %}

### Constraints

{% for c in constraints %}
- {{ c.description }}
$${{ c.formulation }}$$
{% endfor %}

### Objective

{{ objective.description }}
$${{ objective.formulation }}$$
\end{lstlisting}

\begin{lstlisting}[
    style=jinja,
    basicstyle=\scriptsize\ttfamily,
    caption={Parameter Map Template. Rendered as \texttt{map.md} for \texttt{FLARE} and \texttt{\{\{ parameter\_map \}\}} in the \texttt{FLARE-NL} prompt.},
    label={lst:parameter-map-template}
]
# Parameter Map

{% for note in notes %}
{{ note }}

{% endfor %}
{% if definitions %}
## Definitions

{% for name, d in definitions.items() %}
- **{{ name }}**
$${{ d.formulation }}$$
{% endfor %}

{% endif %}
## Parameters

{% for name, d in parameters.items() %}
- **{{ name }}**
$${{ d.formulation }}$$
{% endfor %}
\end{lstlisting}

%% file: appendix/5_formulation_bench.tex
\section{FormulationBench Details}
\label{sec:formulation-bench-details}

The FormulationBench dataset is a collection of 20 optimization problems, 109 MILP formulations, and 89 reformulation pairs (Table \ref{tab:fomulation-bench}). The dataset can be downloaded via the \texttt{formulation-bench} Python package. The documentation\footnote{\url{https://formulation-bench.henryrobbins.com}} provides user guides, detailed problem and formulation descriptions, the dataset schema, and the package API reference. We provide a concise summary here. The dataset is comprised of three sources:

\begin{itemize}[topsep=2pt,parsep=0pt,leftmargin=2em]
    \item \textbf{EquivaFormulation \cite{zhai2025a}.} The five \textit{EquivaFormulation} problems were sampled at random. These are simple optimization problems with a few scalar variables and constraints. Each one contains an original formulation and 10 transformations. These are enumerated in Table 1 of \cite{zhai2025a}. We make the following modifications.
    \begin{itemize}[topsep=2pt,parsep=0pt,leftmargin=2em]
    \item The \emph{Add Valid Inequalities} transformation is instance-specific. Because we are interested in formulation-level reformulation, we omit this transformation type.
    \item We change the label on the \emph{Replace by Base-10 Representation}. This formulation replaces each integer variable with a base-10 decimal expansion. It relies on using enough digit variables to support the size of the test instance data. Hence, this transformation is not valid under our formulation-level notion of reformulation.
    \end{itemize}
    \item \textbf{EvoCut \cite{yazdani2025}.} They define 7 optimization problems for which their method EvoCut proposes numerous cutting planes. We construct reformulation pairs by pairing the original MILP formulation with one augmented by the cut. A cutting plane valid for all instances is a valid reformulation.
    \item \textbf{\citet{ferchtandiker2025}.} They define 8 real-world optimization problems, each admitting an efficient and inefficient formulation. We add each as a pair to \textit{FormulationBench}.
\end{itemize}

\input{tables/problems}

\paragraph*{Dataset Structure.}
Each problem contains (1) a Markdown problem description file, (2) a JSON information file, and (3) JSON files with instance data and the corresponding optimal solution. Each problem contains at least two formulations. Each formulation contains (1) a JSON information file, (2) a Python script to transform raw problem data into the formulation's parameter space, and (3) a ground-truth \texttt{MILPFormulation} Lean formalization. The JSON information file format is an extension of the format introduce by the NLP4LP dataset \cite{ahmaditeshnizi2024}. We add \texttt{assumptions} (see Section \ref{sec:formulation-bench} for a discussion) and \texttt{definitions} fields. Lastly, each formulation is flagged as valid if it is faithful to the underlying optimization problem.

\paragraph*{Reformulation Test Set.}
FormulationBench contains 89 pairs of formulations $(\mathcal{M}, \mathcal{M}')$ containing 63 positive examples where $\mathcal{M}'$ is a constructive reformulation of $\mathcal{M}$ and 26 negative examples. Our experiments use the 54 pairs (42 positive, 12 negative) belonging to the 16 NP-hard problems. Each reformulation pair additionally contains a JSON file specifying the parameter mapping $\Phi_{\mathrm{p}}$ from $\mathcal{P}$ to $\mathcal{P}'$. For every valid reformulation, we provide a ground-truth \texttt{MILPReformulation} proof.

%% file: tables/problems.tex
\begin{table}[th]
\centering
\captionsetup{font=footnotesize, labelfont=bf}
\caption{The 20 optimization problems in the FormulationBench dataset with their source. The \emph{NP-hard} column indicates if the problem is NP-hard. We only use the 16 NP-hard problems in our experiments.}
\label{tab:fomulation-bench}
\footnotesize
\begin{tabular}{llp{3.5cm}c}
\toprule
\textbf{Problem} & \textbf{Name} & \textbf{Source} & \textbf{NP-hard} \\
\midrule
\texttt{p1}  & EquivaFormulation Instance 47 & EquivaFormulation \cite{zhai2025a} & \ding{55} \\
\texttt{p2}  & EquivaFormulation Instance 74 & EquivaFormulation \cite{zhai2025a} & \ding{51} \\
\texttt{p3}  & EquivaFormulation Instance 92 & EquivaFormulation \cite{zhai2025a} & \ding{51} \\
\texttt{p4}  & EquivaFormulation Instance 183 & EquivaFormulation \cite{zhai2025a} & \ding{55} \\
\texttt{p5}  & EquivaFormulation Instance 217 & EquivaFormulation \cite{zhai2025a} & \ding{55} \\
\texttt{p6}  & Capacitated Warehouse Location (CWLP) & EvoCut \cite{yazdani2025} & \ding{51} \\
\texttt{p7}  & Rectangular Tiling (IMO6) & EvoCut \cite{yazdani2025} & \ding{55} \\
\texttt{p8}  & Job Shop Scheduling (JSSP) & EvoCut \cite{yazdani2025} & \ding{51} \\
\texttt{p9}  & Multi-Commodity Network Design (MCND) & EvoCut \cite{yazdani2025} & \ding{51} \\
\texttt{p10} & Pickup and Delivery with Time Windows (PDPTW) & EvoCut \cite{yazdani2025} & \ding{51} \\
\texttt{p11} & Sub-Hour Unit Commitment (SHUC) & EvoCut \cite{yazdani2025} & \ding{51} \\
\texttt{p12} & Traveling Salesman Problem (TSP) & EvoCut \cite{yazdani2025} & \ding{51} \\
\texttt{p13} & Air Traffic Flow Management & \citet{ferchtandiker2025} & \ding{51} \\
\texttt{p14} & Blood Bank Netherlands & \citet{ferchtandiker2025} & \ding{51} \\
\texttt{p15} & Dutch Housing Problem & \citet{ferchtandiker2025} & \ding{51} \\
\texttt{p16} & Park and Bike Hub Location (Mobian) & \citet{ferchtandiker2025} & \ding{51} \\
\texttt{p17} & Open-Pit Mine Production Scheduling & \citet{ferchtandiker2025} & \ding{51} \\
\texttt{p18} & Timor-Leste Hospital Location & \citet{ferchtandiker2025} & \ding{51} \\
\texttt{p19} & UN Humanitarian Disaster Response Hub (UNHDR) & \citet{ferchtandiker2025} & \ding{51} \\
\texttt{p20} & World Food Program Food Distribution & \citet{ferchtandiker2025} & \ding{51} \\
\bottomrule
\end{tabular}
\end{table}

%% file: appendix/6_invalid_reformulations.tex
\section{Invalid Reformulations}
\label{sec:invalid-reformulations}

While preparing the FormulationBench dataset, we verified prior LLM-generated MILP reformulations: cutting planes proposed by EvoCut~\cite{yazdani2025} and reformulation pairs proposed by \citet{ferchtandiker2025}. In the process, \texttt{FLARE} failed to produce reformulation certificates for 5 cutting planes and 4 formulations that were confirmed invalid upon manual inspection. In this section, we provide proofs of invalidity.

\subsection{EvoCut Cutting Planes}
\label{sec:evocut}

\citet{yazdani2025} recently proposed the EvoCut framework to automatically generate \emph{acceleration cuts} for MILP formulations using an LLM-based evolutionary search. Acceleration cuts are constraints added to a MILP formulation with the aim of reducing solve time. Importantly, such cuts are \emph{not} guaranteed to be valid cutting planes or even optimality preserving. This pragmatic choice to consider acceleration cuts enables automation by reducing the computational burden required to verify candidate cuts. A cut is considered an acceleration cut if it does not change the optimal objective value on a small verification set of problem instances.

In the FormulationBench dataset, we constructed a reformulation from a cutting plane by adding the cut to the base formulation. A formulation constructed from a valid cutting plane trivially satisfies Definition \ref{def:constructive}. Using \texttt{FLARE}, we identified that 5 of the 43 acceleration cuts proposed by EvoCut are invalid cutting planes. These cuts span the TSP and a rectangle tiling problem. In the case of TSP, all three invalid cutting plane families are only invalid on TSP instances of size $n\leq 3$. While these counter-examples are not practically meaningful, they illustrate the importance of explicitly stating all assumptions of the problem data. In the case of rectangle tiling, the cuts meaningfully change the set of optimal solutions (see Figure \ref{fig:imo6}). This motivates the importance of formally verifying reformulations, especially for non-standard optimization problems where an LLM is more likely to have an error in reasoning.

\subsubsection{Traveling Salesman Problem (TSP)}

For the TSP, EvoCut generates acceleration cuts for the Miller--Tucker--Zemlin (\ref{eq:mtz}) formulation defined in Section \ref{sec:formulations}. Between \texttt{v1} and \texttt{v2} of the arXiv preprint, there are 8 acceleration cuts proposed. \texttt{FLARE} identifies the following three as invalid cutting planes. Both cuts \ref{eq:tsp-v1-ec3} and \ref{eq:tsp-v2-ec2} eliminate triangles including the depot node and cut \ref{eq:tsp-v2-ec1} eliminates all two-city subtours.
\subimport{formulations/}{tsp_cuts}

\subimport{proofs/}{tsp_v1_ec3_invalid}

\subimport{proofs/}{tsp_v2_ec1_invalid}

\subimport{proofs/}{tsp_v2_ec2_invalid}

\subsubsection{Rectangular Tiling with One Hole per Row and Column (IMO6)}
\label{sec:invalid-imo6}

\subimport{formulations/}{imo6}

\begin{figure}[b]
\centering
\includegraphics[width=0.8\linewidth]{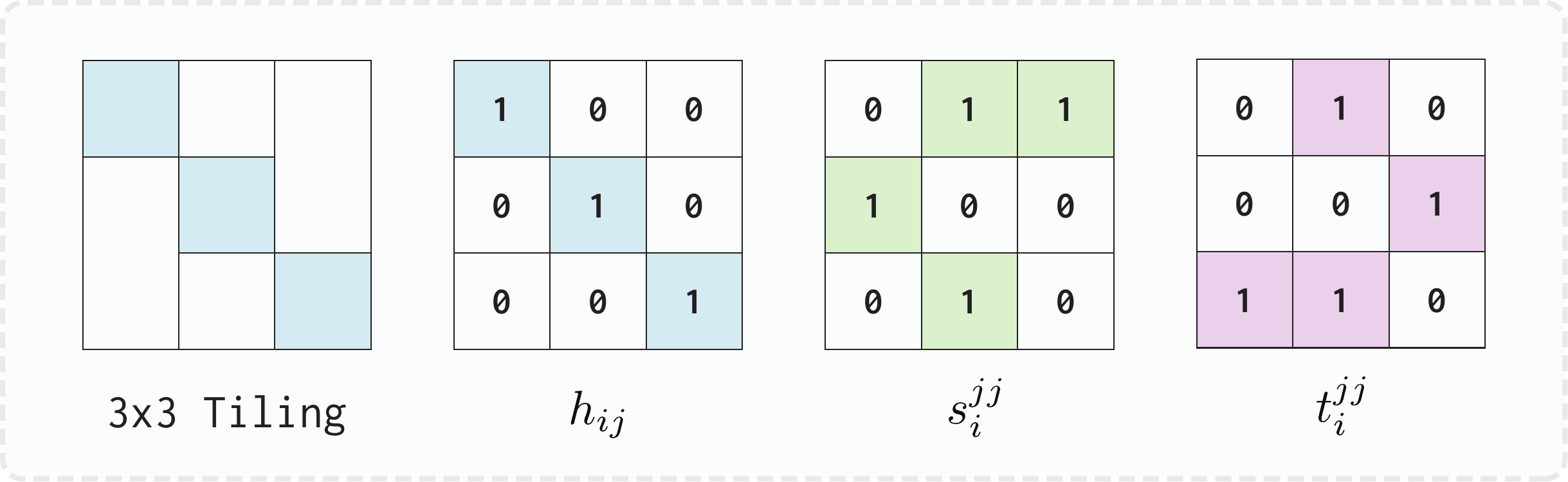}
\captionsetup{font=footnotesize, labelfont=bf}
\caption{A feasible rectangular tiling for the $N=3$ instance of IMO 2025 Problem 6. Decision variable values for the hole indicator $h_{ij}$, start flag $s^{jj}_i$ and end flag $t^{jj}_i$ are depicted in the grids. Note the figure does not specify $s^{ab}_i$ and $t^{ab}_i$ for $a < b$. However, since all tiles have unit width, these variables all have value 0.}
\label{fig:imo6}
\end{figure}

\subimport{proofs/}{imo6_v1_ec1_invalid}

\subimport{proofs/}{imo6_v1_ec2_invalid}

\subsection{Ferchtandiker Formulation Pairs}

\citet{ferchtandiker2025} introduces a dataset\footnote{The dataset is available at \url{https://github.com/nathan-ferchtandiker/LLMs-For-Optimization-Reformulations}} containing an efficient and inefficient MILP formulation for 8 optimization problems. The dataset includes both \LaTeX ~and GurobiPy code for each formulation. In the process of incorporating these reformulations in to the FormulationBench dataset, \texttt{FLARE} identified 4 invalid reformulations.

We prove the Air Traffic Management (ATM) and UN Humanitarian Disaster Response (UNHDR) reformulation pairs are not related by a constructive reformulation in either direction under the identity parameter mapping. We do not provide proofs for the World Food Program (WFP) since both formulations are polynomially-solvable linear programs, and our constructive reformulation definition is only meaningful on NP-hard formulations (Remark \ref{rem:poly-time}). The Open Pit Mining reformulation is intentionally a relaxation. It is expected to violate our definition and we omit any further discussion.

In FormulationBench, we modify the ATM, UNHDR, and WFP formulations to make them valid constructive reformulations, and we modify WFP to be NP-hard. See the documentation\footnote{\url{https://formulation-bench.henryrobbins.com/en/latest/problems}} for detailed descriptions of these modifications.

\subsubsection{Air Traffic Management (ATM)}

\subimport{formulations/}{atm}

The two formulations have misaligned objectives semantics. In the inefficient formulation, the reward $\sum_{p,a,t} r_{a,t} y_{p,a,t}$ accrues at every time period a plane is present at a location, so a plane that remains at $a$ for $k$ consecutive time periods contributes $\sum_{t} r_{a,t}$ over those $k$ periods. In the efficient formulation, the reward only accrues on arrivals, so the same solution contributes only the single reward $r_{a,t}$ at the arrival time.

First, we establish a useful obstruction that asserts the existence of a bijection between the objective value sets of constructive reformulations.

\subimport{lemmas/}{c_objective_level_bijection}

\subimport{proofs/}{c_objective_level_bijection}

\subimport{proofs/}{atm_not_constructive}

\subsubsection{UN Humanitarian Disaster Response Hub (UNHDR)}

\subimport{formulations/}{unhdr}

The two formulations differ in how the demand of a region may be met. In the \ref{eq:unhdr-inefficient} formulation, the assignment constraint $\sum_{h \in H} z_{hc} = 1$ together with the big-$M$ constraint $x_{hc} \leq M z_{hc}$ forces each region $c \in C$ to be served in full by a single hub. The \ref{eq:unhdr-efficient} formulation has no such restriction; a region can be supplied by a combination of multiple hubs simultaneously.

\subimport{proofs/}{unhdr_not_constructive}

\subsubsection{World Food Program (WFP)}

\subimport{formulations/}{wfp}

Both WFP formulations are polynomially-solvable linear programs, but our constructive reformulation definition is only meaningful on NP-hard formulations (Remark \ref{rem:poly-time}). Despite this, \texttt{FLARE} still identified two inconsistencies between the pair of WFP formulations.

\begin{itemize}[topsep=2pt,parsep=0pt,leftmargin=2em]
    \item \textbf{Misaligned objectives.} In the \ref{eq:wfp-efficient} formulation, the procurement cost term $\sum_k \mathrm{pc}_k \sum_{j \in N_B} \mathrm{dem}_j R_k$ depends only on the ration size $R_k$, while the demand constraint $\sum_i E_{ij} F_{ijk} \geq \mathrm{dem}_j R_k$ permits shipping in excess of $\mathrm{dem}_j R_k$ without additional procurement penalty. In the \ref{eq:wfp-inefficient} formulation, the procurement cost $\sum_k \mathrm{pc}_k \sum_p x_{pk}$ instead depends on the total amount shipped, so any slack in the demand constraint is charged.
    \item \textbf{Cycles.} In the \ref{eq:wfp-efficient} formulation, we have flow variables $F_{ijk}$. A feasible point could have positive flow on a cycle of transshipment points. However, in the \ref{eq:wfp-inefficient} formulation, $P$ is the set of all \emph{simple} paths from a supplier to a beneficiary camp. A feasible point necessarily has acyclic flow on each commodity.
\end{itemize}

Notably, these inconsistencies are only realized on sub-optimal feasible points. An optimal point will not ship excess demand nor send flow on a cycle. While the \ref{eq:wfp-efficient} formulation is an Audet reformulation of the \ref{eq:wfp-inefficient} formulation, our stricter constructive reformulation definition imposes conditions across the \emph{entire} feasible region.

We modify WFP to resolve these inconsistencies for inclusion in FormulationBench and make the problem NP-hard by restricting to integral flows and adding transhipment throughput capacities.

%% file: formulations/tsp_cuts.tex
\begin{align*}
& x_{j1} + x_{ji} + (u_j - u_i - 1) \le (n-1)(2 - x_{1i} - x_{ij}) && \forall i,j \in V \setminus \{1\},\; i \ne j \label{eq:tsp-v1-ec3}\tag{\texttt{v1}-EC3}\\
& x_{ij} + x_{ji} \le 1 && \forall i,j \in V,\; i < j \label{eq:tsp-v2-ec1}\tag{\texttt{v2}-EC1}\\
& x_{1i} + x_{i1} + x_{1j} + x_{j1} + x_{ij} + x_{ji} \le 2 && \forall i,j \in V \setminus \{1\},\; i < j \label{eq:tsp-v2-ec2}\tag{\texttt{v2}-EC2}\\
\end{align*}

%% file: proofs/tsp_v1_ec3_invalid.tex
\begin{proposition}
\label{prop:tsp-v1-ec3}
The cut \ref{eq:tsp-v1-ec3} is \emph{not} a valid cutting plane for the \ref{eq:mtz} TSP formulation.
\end{proposition}
\begin{proof}
Consider the 3-node TSP instance depicted in Figure \ref{fig:tsp} with feasible tour $1 \to 2 \to 3 \to 1$. Letting $i = 2$ and $j = 3$, the cut \ref{eq:tsp-v1-ec3} reduces to $1 \le 0$, which does not hold. Hence, there exists a feasible integer point that is not satisfied by the cut. Therefore, the cut is invalid.
\end{proof}

%% file: proofs/tsp_v2_ec1_invalid.tex
\begin{proposition}
\label{prop:tsp-v2-ec1}
The cut \ref{eq:tsp-v2-ec1} is \emph{not} a valid cutting plane for the \ref{eq:mtz} TSP formulation.
\end{proposition}
\begin{proof}
Consider the 2-node TSP instance with the unique feasible tour $1 \to 2 \to 1$. The MTZ formulation permits \(x_{12}=x_{21}=1\), corresponding to traversing both arcs of this tour. Setting \(i=1\) and \(j=2\), the cut \ref{eq:tsp-v2-ec1} reduces to \(2 \le 1\), which does not hold. Hence, there exists a feasible integer point that is not satisfied by the cut. Therefore, the cut is invalid.
\end{proof}

%% file: proofs/tsp_v2_ec2_invalid.tex
\begin{proposition}
\label{prop:tsp-v2-ec2}
The cut \ref{eq:tsp-v2-ec2} is \emph{not} a valid cutting plane for the \ref{eq:mtz} TSP formulation.
\end{proposition}
\begin{proof}
Consider the 3-node TSP instance depicted in Figure \ref{fig:tsp} with feasible tour $1 \to 2 \to 3 \to 1$. Setting \(i=2\) and \(j=3\), the cut \ref{eq:tsp-v2-ec2} reduces to \(3 \le 2\), which does not hold. Hence, there exists a feasible integer point that is not satisfied by the cut. Therefore, the cut is invalid.
\end{proof}

%% file: formulations/imo6.tex
This problem is inspired by IMO 2025 Problem 6. Given an $N \times N$ grid of unit squares, one must place rectangular tiles (of various sizes) such that each row and column of the grid contain exactly one uncovered square (a \emph{hole}). The objective is to minimize the number of tiles used. \citet{yazdani2025} introduce the following notation and formulation.

\paragraph*{Notation.}
\begin{itemize}[topsep=2pt,parsep=0pt,leftmargin=2em]
    \item $R = \{1,\dots,N\}$ rows and $C = \{1,\dots,N\}$ columns
    \item $I = \{(a,b) \in C^2 ~|~ a \leq b\}$ contiguous column-intervals
    \item $h_{ij} \in \{0,1\}$ hole indicator
    \item $x^{ab}_i \in \{0,1\}$ indicates if columns $a$ through $b$ on row $i$ are covered by the same tile
    \item $s^{ab}_i,t^{ab}_i\in \{0,1\}$ indicate if a tile spanning columns $a$ to $b$ begins on row $i$ or ends on row $i$, respectively
\end{itemize}

\paragraph*{Formulation.}
\footnotesize
\begin{equation*}
\begin{aligned}
\min\;&\sum_{i\in R}\sum_{(a,b)\in I} s_i^{ab}\\[3pt]
\text{s.t.}\;
&\sum_{j\in C} h_{ij}=1 &&\forall i\in R\\
&\sum_{i\in R} h_{ij}=1 &&\forall j\in C\\
&\sum_{\substack{(a,b)\in I\\ a\le j\le b}} x_i^{ab}+h_{ij}=1
&&\forall i\in R,\;\forall j\in C\\
&x_1^{ab}-s_1^{ab}=0 &&\forall (a,b)\in I\\
&x_i^{ab}-x_{i-1}^{ab}-s_i^{ab}+t_{i-1}^{ab}=0
&&\forall i=2,\dots,N,\;\forall(a,b)\in I\\
&x_N^{ab}-t_N^{ab}=0 &&\forall (a,b)\in I\\
&h_{ij}\in\{0,1\}
&&\forall i\in R,\;\forall j\in C\\
&x_i^{ab},s_i^{ab},t_i^{ab}\in\{0,1\}
&&\forall i\in R,\;\forall(a,b)\in I
\end{aligned}
\label{eq:imo6}
\tag{IMO6}
\end{equation*}
\normalsize

Between \texttt{v1} and \texttt{v2} of the arXiv preprint, there are 8 acceleration cuts proposed. \texttt{FLARE} identifies the following two as invalid cutting planes. The cut \ref{eq:imo6-v1-ec1} enforces that a tile end (bottom-right corner) must be immediately to the left of any hole, and the cut \ref{eq:imo6-v1-ec2} enforces that a tile start (top-left corner) must be immediately to the right of any hole.

\footnotesize
\begin{align*}
& h_{ij} \le \sum_{\substack{(a,b)\in I\\ b=j-1}} t_i^{ab} && \forall i\in R,\; \forall j\in\{2,\dots,N\} \label{eq:imo6-v1-ec1}\tag{\texttt{v1}-EC1}\\
& h_{ij} \le \sum_{\substack{(a,b)\in I\\ a=j+1}} s_i^{ab} && \forall i\in R,\; \forall j\in\{1,\dots,N-1\} \label{eq:imo6-v1-ec2}\tag{\texttt{v1}-EC2}\\
\end{align*}
\normalsize

%% file: proofs/imo6_v1_ec1_invalid.tex
\begin{proposition}
\label{prop:imo6-v1-ec1}
The cut \ref{eq:imo6-v1-ec1} is \emph{not} a valid cutting plane for the \ref{eq:imo6} formulation.
\end{proposition}
\begin{proof}
Consider the $N=3$ counterexample depicted in Figure \ref{fig:imo6}. Let $i=j=2$. The inequality reduces to $h_{22} \leq t^{11}_2$. Substituting yields $1 \leq 0$, which does not hold. Hence, there exists a feasible integer point that is not satisfied by the cut. Therefore, the cut is invalid.
\end{proof}

%% file: proofs/imo6_v1_ec2_invalid.tex
\begin{proposition}
\label{prop:imo6-v1-ec2}
The cut \ref{eq:imo6-v1-ec2} is \emph{not} a valid cutting plane for the \ref{eq:imo6} formulation.
\end{proposition}
\begin{proof}
Consider the $N=3$ counterexample depicted in Figure \ref{fig:imo6}. Let $i=2$ and $j=2$. The inequality reduces to $h_{22} \leq s^{33}_2$. Substituting yields $1 \leq 0$, which does not hold. Hence, there exists a feasible integer point that is not satisfied by the cut. Therefore, the cut is invalid.
\end{proof}

%% file: formulations/atm.tex
This problem consists of assigning each plane in an airline's fleet to a location (airport or airspace sector) over a sequence of time periods. The assignment must respect each location's capacity over time and the network's adjacency structure. Planes can only travel to locations that are a single time unit away. The objective is to maximize the total reward accrued for visiting locations. \citet{ferchtandiker2025} introduces the following formulations of this problem.

\paragraph*{Notation.}
\begin{itemize}[topsep=2pt,parsep=0pt,leftmargin=2em]
    \item $P$: set of vehicles (planes).
    \item $A$: set of locations (airports and airspace sectors).
    \item $T$: set of time periods.
    \item $\mathrm{cap}_{a,t}$: capacity of location $a \in A$ at time $t \in T$.
    \item $\tau_{a,a'}$: travel time from location $a$ to location $a'$.
    \item $r_{a,t}$: reward for being at location $a \in A$ at time $t \in T$.
\end{itemize}

\paragraph*{Formulations.}
Both formulations share the binary departure variable $x_{p,a,a',t} \in \{0,1\}$, which indicates whether plane $p$ departs from $a$ to $a'$ at time $t$.\footnote{In the dataset, this variable is denoted $z_{p,a,a',t}$ in the inefficient formulation. We use $x_{p,a,a',t}$ for both formulations.} The \ref{eq:atm-efficient} formulation is purely event-based. The \ref{eq:atm-inefficient} formulation additionally tracks plane locations via $y_{p,a,t} \in \{0,1\}$, which indicates whether plane $p$ is at location $a$ at time $t$.

\noindent\tiny%
\begin{minipage}[t]{0.48\textwidth}
\vspace{-8pt}
\begin{equation}
\begin{aligned}
\max \quad & \sum_{p \in P} \sum_{a \in A} \sum_{t \in T} r_{a,t} y_{p,a,t} \\
\text{s.t.} \quad & \sum_{a \in A} y_{p,a,t} = 1 && \forall p \in P,\; t \in T \\
& \sum_{p \in P} y_{p,a,t} \leq \mathrm{cap}_{a,t} && \forall a \in A,\; t \in T \\
& y_{p,a,t} = y_{p,a,t-1} + \\ & \quad \sum_{a' \in A} \sum_{t' : t' + \tau_{a',a} = t} x_{p,a',a,t'} \\ & \quad - \sum_{a' \in A} x_{p,a,a',t} && \forall p \in P,\; a \in A,\; t > 0 \\
& y_{p,a,t} \in \{0,1\} && \forall p \in P,\; a \in A,\; t \in T \\
& x_{p,a,a',t} \in \{0,1\} && \forall p \in P,\; a,a' \in A,\; t \in T \\
\end{aligned}
\tag{Inefficient}\label{eq:atm-inefficient}
\end{equation}
\end{minipage}%
\hfill%
\begin{minipage}[t]{0.48\textwidth}
\vspace{-8pt}
\begin{equation}
\begin{aligned}
\max \quad & \sum_{p \in P} \sum_{a \in A} \sum_{t \in T} r_{a,t} \\ & \quad \left( \sum_{a' \in A} \sum_{t' : t' + \tau_{a',a} = t} x_{p,a',a,t'} \right) \\
\text{s.t.} \quad & \sum_{a \in A} \sum_{a' \in A} \sum_{t \in T} x_{p,a,a',t} \geq 1 && \forall p \in P \\
& \sum_{p \in P} \sum_{a' \in A} x_{p,a,a',t} \leq \mathrm{cap}_{a,t} && \forall a \in A,\; t \in T \\
& x_{p,a,a',t} \in \{0,1\} && \forall p \in P,\; a,a' \in A,\; t \in T \\
\end{aligned}
\tag{Efficient}\label{eq:atm-efficient}
\end{equation}
\end{minipage}\normalsize

%% file: lemmas/c_objective_level_bijection.tex
\begin{restatable}[Objective-level bijection]{lemma}{cobjlevelbijection}
\label{lem:c-objective-level-bijection}
Let $\Phi = (\Phi_{\mathrm{p}}, \Phi_{\mathrm{fwd}}, \Phi_{\mathrm{bwd}}, \Phi_{\mathrm{obj}})$ be a reformulation construction (Definition \ref{def:constructive}) from $\mathcal{M}$ to $\mathcal{M}'$, and let $p \in \mathcal{P}$ with $p' = \Phi_{\mathrm{p}}(p)$. Write $V(p) = \{f_0(x;p) : x \in \mathcal{F}(p)\}$ and $V'(p') = \{f_0'(x';p') : x' \in \mathcal{F}'(p')\}$ for the attained objective-value sets. Then $\Phi_{\mathrm{obj}}$ is a bijection from $V(p)$ onto $V'(p')$; in particular $|V(p)| = |V'(p')|$.
\end{restatable}

%% file: proofs/c_objective_level_bijection.tex
\begin{proof}
We first show $\Phi_{\mathrm{obj}}(V(p)) \subseteq V'(p')$. For $v = f_0(x;p) \in V(p)$ with $x \in \mathcal{F}(p)$, the point $\Phi_{\mathrm{fwd}}(x;p)$ lies in $\mathcal{F}'(p')$ by forward feasibility and has objective value $\Phi_{\mathrm{obj}}(v)$ by objective preservation. Hence $\Phi_{\mathrm{obj}}(v) \in V'(p')$. We now show the reverse inclusion $V'(p') \subseteq \Phi_{\mathrm{obj}}(V(p))$. For $v' = f_0'(x';p') \in V'(p')$ with $x' \in \mathcal{F}'(p')$, the point $x = \Phi_{\mathrm{bwd}}(x';p)$ lies in $\mathcal{F}(p)$ by backward feasibility and satisfies $v' = \Phi_{\mathrm{obj}}(f_0(x;p))$ by objective preservation. Hence $v' \in \Phi_{\mathrm{obj}}(V(p))$. Both inclusions yield $\Phi_{\mathrm{obj}}(V(p)) = V'(p')$. Since $\Phi_{\mathrm{obj}}$ is strictly monotonically increasing, it is injective, so this surjection is a bijection and $|V(p)| = |V'(p')|$.
\end{proof}

%% file: proofs/atm_not_constructive.tex
\begin{proposition}
\label{prop:atm-not-constructive}
The \ref{eq:atm-inefficient} and \ref{eq:atm-efficient} formulations of ATM are not related by a constructive reformulation (Definition~\ref{def:constructive}) in either direction under the identity parameter mapping $\Phi_\textrm{p}$.
\end{proposition}
\begin{proof}

Consider a single-plane, single-location instance $p$. Let $P = \{1\}$, $A = \{a\}$, and $T = \{0,1\}$ with $\tau_{a,a} = 1$, $\mathrm{cap}_{a,0} = \mathrm{cap}_{a,1} = 1$, and $r_{a,0} = r_{a,1} = 1$. Write $x_t := x_{1,a,a,t}$ and $y_t := y_{1,a,t}$. Let $V_\textrm{eff}(p)$ and $V_\textrm{ineff}(p)$ denote the objective values attained on $p$ by the \ref{eq:atm-efficient} and \ref{eq:atm-inefficient} formulations respectively.

In the \ref{eq:atm-efficient} formulation, the constraints reduce to $x_0 + x_1 \geq 1$ and $x_t \leq \mathrm{cap}_{a,t} = 1$, so the feasible set is $\{(1,0),(0,1),(1,1)\}$. A departure at time $t$ arrives at time $t + \tau_{a,a} = t+1$, so a departure at $t = 1$ arrives at $t = 2 \notin T$ and contributes nothing; the objective equals $r_{a,1} x_0 = x_0$. Hence $V_\textrm{eff}(p) = \{0,1\}$.

In the \ref{eq:atm-inefficient} formulation, $\sum_{a \in A} y_{1,a,t} = 1$ with $|A| = 1$ forces $y_0 = y_1 = 1$ at every feasible point. The point $y_0 = y_1 = 1$, $x_0 = x_1 = 0$ satisfies the transition constraint, so the feasible set is nonempty. Its objective is $y_0 r_{a,0} + y_1 r_{a,1} = 2$, giving $V_\textrm{ineff}(p) = \{2\}$.

Suppose that a constructive reformulation between the two formulations with $\Phi_{\mathrm{p}}$ the identity exists. By Lemma \ref{lem:c-objective-level-bijection}, $|V_\textrm{eff}(p)| = |V_\textrm{ineff}(p)|$ regardless of which formulation is the source. But $|V_\textrm{eff}(p)| = 2 \neq 1 = |V_\textrm{ineff}(p)|$, a contradiction.
\end{proof}

%% file: formulations/unhdr.tex
This problem consists of selecting a fixed number of response hubs to service a set of disaster regions. The selection must satisfy demand and response time constraints. The objective is to minimize transportation cost. \citet{ferchtandiker2025} introduces the following formulations of this problem.

\paragraph*{Notation.}
\begin{itemize}[topsep=2pt,parsep=0pt,leftmargin=2em]
    \item $H$: set of candidate hubs.
    \item $H^f \subseteq H$: set of hubs that must be open (fixed hubs).
    \item $C$: set of disaster regions.
    \item $a_c$: number of people affected in region $c \in C$.
    \item $C_{hc}$: cost per person from hub $h$ to region $c$.
    \item $t_{hc}$: travel time from hub $h$ to region $c$.
    \item $T$: maximum allowed (weighted) travel time per region.
    \item $n$: maximum number of hubs that can be opened.
    \item $M$: sufficiently large constant (used in big-M constraints).
\end{itemize}

\paragraph*{Formulations.}
The \ref{eq:unhdr-inefficient} formulation introduces a binary indicator variable $y_h$ to indicate if hub $h \in H$ is opened and binary indicator $z_{hc}$ to indicate if hub $h \in H$ is assigned to serve region $c \in C$. Lastly, $x_{hc}$ is the fraction of region $c$'s demand that is served by hub $h$.\footnote{In the dataset, this variable is denoted $q_{hc}$. We use $x_{hc}$ for consistency with the efficient formulation.} The \ref{eq:unhdr-efficient} formulation drops the hub indicator $z_{hc}$ variables.

\noindent\tiny%
\begin{minipage}[t]{0.48\textwidth}
\vspace{-8pt}
\begin{equation}
\begin{aligned}
\min \quad & \sum_{h \in H} \sum_{c \in C} a_c\,C_{hc}\,x_{hc} \\
\text{s.t.} \quad & x_{hc} \leq M z_{hc} && \forall h \in H,\; c \in C \\
& \sum_{h \in H} z_{hc} = 1 && \forall c \in C \\
& \sum_{c \in C} z_{hc} \leq |C| y_h && \forall h \in H \\
& \sum_{h \in H} x_{hc} = 1 && \forall c \in C \\
& \sum_{h \in H} y_h \leq n \\
& y_h = 1 && \forall h \in H^f \\
& \sum_{h \in H} t_{hc} x_{hc} \leq T && \forall c \in C \\
& x_{hc} \geq 0 && \forall h \in H,\; c \in C \\
& z_{hc} \in \{0,1\} && \forall h \in H,\; c \in C \\
& y_h \in \{0,1\} && \forall h \in H \\
\end{aligned}
\tag{Inefficient}\label{eq:unhdr-inefficient}
\end{equation}
\end{minipage}%
\hfill%
\begin{minipage}[t]{0.48\textwidth}
\vspace{-8pt}
\begin{equation}
\begin{aligned}
\min \quad & \sum_{h \in H} \sum_{c \in C} a_c C_{hc} x_{hc} \\
\text{s.t.} \quad & \sum_{c \in C} x_{hc} \leq |C| y_h && \forall h \in H \\
& \sum_{h \in H} x_{hc} = 1 && \forall c \in C \\
& \sum_{h \in H} y_h \leq n \\
& y_h = 1 && \forall h \in H^f \\
& \sum_{h \in H} t_{hc} x_{hc} \leq T && \forall c \in C \\
& x_{hc} \geq 0 && \forall h \in H,\; c \in C \\
& y_h \in \{0,1\} && \forall h \in H \\
\end{aligned}
\tag{Efficient}\label{eq:unhdr-efficient}
\end{equation}
\end{minipage}\normalsize

%% file: proofs/unhdr_not_constructive.tex
\begin{proposition}
\label{prop:unhdr-not-constructive}
The \ref{eq:unhdr-inefficient} and \ref{eq:unhdr-efficient} formulations of UNHDR are not related by a constructive reformulation (Definition~\ref{def:constructive}) in either direction under the identity parameter mapping $\Phi_\textrm{p}$.
\end{proposition}
\begin{proof}
Consider the instance $p$ with $H = \{1, 2\}$, $H^f = \emptyset$, $C = \{c\}$, $a_c = 1$, $n = 2$, costs $C_{1c} = 0$, $C_{2c} = 1$, travel times $t_{1c} = 2$, $t_{2c} = 0$, response time limit $T = 1$, and any $M \geq 1$. Let $V_\textrm{eff}(p)$ and $V_\textrm{ineff}(p)$ denote the objective values attained on $p$ by the \ref{eq:unhdr-efficient} and \ref{eq:unhdr-inefficient} formulations respectively.

In the \ref{eq:unhdr-efficient} formulation, open both hubs ($y_1 = y_2 = 1$). The point $x_{1c} = \lambda$, $x_{2c} = 1 - \lambda$ is feasible for every $\lambda \in [0, \tfrac{1}{2}]$ (the response time constraint reads $2\lambda \leq 1$) and attains objective value $1 - \lambda$. Hence $V_\textrm{eff}(p) \supseteq [\tfrac{1}{2}, 1]$ is uncountable.

In the \ref{eq:unhdr-inefficient} formulation, assigning region $c$ to hub $1$ ($z_{1c} = 1$) forces $x_{1c} = 1$ and violates the response time constraint $t_{1c} x_{1c} = 2 > 1 = T$. Hence the only feasible assignment is $z_{2c} = 1$, which forces $x_{2c} = 1$ and $y_2 = 1$, attaining objective value $1$. Thus $V_\textrm{ineff}(p) = \{1\}$.

Suppose that a constructive reformulation between the two formulations with $\Phi_{\mathrm{p}}$ the identity exists. By Lemma \ref{lem:c-objective-level-bijection}, $|V_\textrm{eff}(p)| = |V_\textrm{ineff}(p)|$ regardless of which formulation is the source. But $V_\textrm{eff}(p)$ is uncountable while $|V_\textrm{ineff}(p)| = 1$, a contradiction.
\end{proof}

%% file: formulations/wfp.tex
This problem consists of designing a food distribution plan for the World Food Program (WFP) to deliver commodities from suppliers, through transshipment points, to beneficiary camps in crisis-affected regions. The plan must satisfy each camp's ration demand and meet per-person nutritional requirements across all nutrients. The objective is to minimize the total procurement and transportation cost. \citet{ferchtandiker2025} introduces the following formulations of this problem.

\paragraph*{Notation.}
\begin{itemize}[topsep=2pt,parsep=0pt,leftmargin=2em]
    \item $K$: set of commodities.
    \item $L$: set of nutrients.
    \item $\mathrm{pc}_k$ is the procurement cost per kg of commodity $k$.\footnote{In the dataset, the variable $q_k$ is used for procurement cost in the inefficient formulation. We use $\mathrm{pc}_k$ for both formulations.}
    \item $\mathrm{nutval}_{k\ell}$: nutrient-$\ell$ content (per kg) of commodity $k \in K$.
    \item $\mathrm{nutreq}_\ell$: per-person requirement for nutrient $\ell \in L$.
    \item $\mathrm{dem}_j$: number of beneficiaries at camp $j$.
    \item $R_k \geq 0$: ration size (kg per person) of commodity $k \in K$.
\end{itemize}

\paragraph*{Formulations.}
The \ref{eq:wfp-efficient} formulation is node-based: $N$ is the set of nodes, partitioned into suppliers $N_S$, transshipment points $N_T$, and beneficiary camps $N_B$. $E_{ij} \in \{0,1\}$ indicates whether an edge from $i$ to $j$ exists and we have a transportation cost $\mathrm{tc}_{ijk} \geq 0$ per kg of commodity $k$ on edge $(i,j)$. Edges run only from suppliers to transshipment points, between transshipment points, and from transshipment points to beneficiary camps; that is, $E_{ij} = 1$ only if
\[
(i,j) \in (N_S \times N_T) \;\cup\; (N_T \times N_T) \;\cup\; (N_T \times N_B).
\]
The decision variable $F_{ijk} \geq 0$ encodes the amount of commodity $k$ shipped from $i$ to $j$. The \ref{eq:wfp-inefficient} formulation is path-based: $P$ is the set of all simple paths from a supplier to a beneficiary camp. The indicator $e_{jp} \in \{0,1\}$ indicates whether path $p$ ends at camp $j$ and $c_{pk}$ is the cost of shipping one kg of commodity $k$ along path $p$. The decision variable $x_{pk} \geq 0$ is the amount of commodity $k$ shipped along path $p$.

\noindent\tiny%
\begin{minipage}[t]{0.48\textwidth}
\vspace{-8pt}
\begin{equation}
\begin{aligned}
\min \quad & \sum_{k \in K} pc_k \left( \sum_{p \in P} x_{pk} \right) +
\\ & \quad \sum_{p \in P} \sum_{k \in K} c_{pk}\, x_{pk} \\
\text{s.t.} \quad & \sum_{p \in P} e_{jp}\, x_{pk} \geq \mathrm{dem}_j\, R_k && \forall j \in N_B,\; k \in K \\
& \sum_{k \in K} \mathrm{nutval}_{k\ell}\, R_k \geq \mathrm{nutreq}_\ell && \forall \ell \in L \\
& x_{pk} \geq 0 && \forall p \in P,\; k \in K \\
& R_k \geq 0 && \forall k \in K \\
\end{aligned}
\tag{Inefficient}\label{eq:wfp-inefficient}
\end{equation}
\end{minipage}%
\hfill%
\begin{minipage}[t]{0.48\textwidth}
\vspace{-8pt}
\begin{equation}
\begin{aligned}
\min \quad & \sum_{k \in K} \mathrm{pc}_k \left( \sum_{j \in N_B} \mathrm{dem}_j R_k \right) + \\ & \quad \sum_{i,j \in N} \sum_{k \in K} \mathrm{tc}_{ijk}\, F_{ijk} \\
\text{s.t.} \quad & \sum_{i \in N} E_{ij} F_{ijk} = \sum_{i \in N} E_{ji} F_{jik} && \forall j \in N_T,\; k \in K \\
& \sum_{i \in N} E_{ij} F_{ijk} \geq \mathrm{dem}_j R_k && \forall j \in N_B,\; k \in K \\
& \sum_{k \in K} \mathrm{nutval}_{k\ell} R_k \geq \mathrm{nutreq}_\ell && \forall \ell \in L \\
& F_{ijk} \geq 0 && \forall i,j \in N,\; k \in K \\
& R_k \geq 0 && \forall k \in K \\
\end{aligned}
\tag{Efficient}\label{eq:wfp-efficient}
\end{equation}
\end{minipage}\normalsize

%% file: appendix/7_experimental_details.tex
\section{Experiment Details}
\label{sec:experimental-details}

\paragraph*{Dataset.}
All experiments were run on \texttt{v0.5.0} of the \texttt{formulation-bench} Python package which pins the dataset to \texttt{v0.4.0}. See Appendix \ref{sec:formulation-bench-details} and the package documentation for details.

\paragraph*{\texttt{FLARE} agent harness.}
Experiments were conducted on a MacBook Pro (Apple M3 Pro, 12-core CPU, 18 GB unified memory) running macOS 15.7.1. We evaluate \texttt{FLARE} on the following agent harnesses:

\begin{itemize}
    \item \textbf{Claude Code.} Version 2.1.197 with Claude Code Max subscription (\$100/month)
    \item \textbf{Codex.} Version 0.147.0 with ChatGPT Pro subscription (\$100/month)
    \item \textbf{OpenCode.} Version 1.18.15 (open-source)
\end{itemize}

We used Claude Code and ChatGPT subscriptions in order to avoid higher API costs. The \texttt{FLARE} experiments can be run within the weekly usage limits but must be batched in order to avoid the 5 hour Claude Code session limit. 

\paragraph*{\texttt{FLARE} compute.}
\texttt{FLARE} runs in a Docker container with a 30 minute time limit. The image is built on Ubuntu 24.04 with Python 3.12 and Node 20. It contains the three agent CLIs above, the Lean toolchain (\texttt{elan} 4.2.3 and Lean 4.28.0, with \texttt{mathlib} pinned to \texttt{v4.28.0}), and the \texttt{lean-lsp-mcp} 0.29.0 MCP server. To allow for increased parallelization, we run each \texttt{FLARE} in its own Modal\footnote{\url{https://modal.com}} Sandbox. Each Sandbox is allocated a guaranteed floor of 2 CPU cores (it may burst higher) and 4 GiB of memory. Modal compute costs are excluded from the reported average costs. The compute costs of the \texttt{FLARE} experiments falls below the \$30/month compute provided for free.

\input{tables/pricing}

\paragraph*{LLMs.} 
Anthropic and OpenAI models are served through the Stanford AI API Gateway\footnote{\url{https://uit.stanford.edu/service/ai-api-gateway}} at a discounted price. The reported average costs are computed using the standard API rates at the time of experimentation (Table \ref{tab:pricing}). The gateway only supports structured output on OpenAI models (if reasoning is enabled). For Anthropic and DeepSeek models, we append the response schema to the prompt and configure retry logic if the response isn't parseable. We specify a max token budget of 8192 when reasoning is disabled and 16384 when reasoning is enabled. We use medium effort for Anthropic and OpenAI models and high effort for DeepSeek models. We found these settings produce comparable reasoning efforts across model families. DeepSeek models reason for longer, but \emph{high} is the lowest effort available. When a model exhausts its token budget before emitting a verdict, we count the truncated response as a judgment of not a reformulation; this occurs only for DeepSeek V4 Pro, on 22 of its 486 ablation runs.

%% file: tables/pricing.tex
\begin{table}[ht]
\centering
\captionsetup{font=footnotesize, labelfont=bf}
\caption{Standard API rates (USD per million tokens) used to compute the reported average costs. \emph{Input} is the uncached (cache-miss) rate and \emph{Cached Input} the cache-hit rate; cache writes (billed at 1.25$\times$ input by Anthropic) are not modeled. GPT-4.1 publishes no cached rate, so its cached tokens are billed at the input rate. DeepSeek uses preview release rates prior to the 8-16-26 price increase.}
\label{tab:pricing}
\footnotesize
\setlength{\tabcolsep}{6pt}
\begin{tabular}{llrrr}
\toprule
\textbf{Model} & \textbf{API Identifier} & \textbf{Input} & \textbf{Cached Input} & \textbf{Output} \\
\midrule
Opus 5            & \texttt{claude-opus-5}     & \$5.000  & \$0.5000  & \$25.00 \\
Sonnet 5          & \texttt{claude-sonnet-5}   & \$2.000  & \$0.2000  & \$10.00 \\
GPT-5.6 Sol       & \texttt{gpt-5.6-sol}       & \$5.000  & \$0.5000  & \$30.00 \\
GPT-5.6 Terra     & \texttt{gpt-5.6-terra}     & \$2.000  & \$0.2000  & \$12.00 \\
GPT-4.1           & \texttt{gpt-4.1}           & \$2.000  & ---       & \$8.00  \\
DeepSeek V4 Pro   & \texttt{deepseek-v4-pro}   & \$0.435  & \$0.0036  & \$0.87  \\
DeepSeek V4 Flash & \texttt{deepseek-v4-flash} & \$0.140  & \$0.0028  & \$0.28  \\
\bottomrule
\end{tabular}
\end{table}

%% file: appendix/8_additional_results.tex
\section{Additional Results}
\label{sec:addtional-results}

\input{tables/flare.tex}
\input{tables/models_p12}

%% file: tables/flare.tex
\begin{table}[h]
\centering
\captionsetup{font=footnotesize, labelfont=bf}
\caption{Evaluation of \texttt{FLARE} across agent harnesses and LLM models. Accuracy results on the FormulationBench dataset are provided in aggregate and segmented by source dataset. Metric cells report results from a single run and \textbf{Bold} indicates the best method on a metric. The Claude Code harness with Opus 5 is the only configuration with \textbf{100\%} accuracy. Codex with GPT-5.6 Sol misses one pair. The open-source OpenCode harness with DeepSeek V4 Pro is an order of magnitude cheaper, but is substantially slower and less accurate. All 11 false negatives are attributable to ATP failure, 8 due to timeout limits (Appendix \ref{sec:experimental-details}).}
\label{tab:flare}
\scriptsize
\setlength{\tabcolsep}{4pt}
\begin{tabular}{l@{\hskip 4pt}l@{\hskip 4pt}lrrrrrrrrr}
\toprule
\textbf{Harness} & \textbf{Model} & \textbf{Effort} & \textbf{TP} & \textbf{FP} & \textbf{TN} & \textbf{FN} & \textbf{Precision} & \textbf{Recall} & \textbf{Accuracy} & \textbf{Avg. Time} & \textbf{Avg. Cost} \\
\midrule
\rowcolor{black!10}\multicolumn{12}{l}{\emph{Overall}} \\
\midrule
Claude Code & Opus 5          & medium &  42 &  0 & 12 &   0 & \cellcolor{light-blue!50}\textbf{100.0\%} & \cellcolor{light-blue!50}\textbf{100.0\%} & \cellcolor{light-blue!50}\textbf{100.0\%} &  410.6s & \$1.180 \\
Codex       & GPT-5.6 Sol     & medium &  41 &  0 & 12 &   1 & \cellcolor{light-blue!50}\textbf{100.0\%} & \cellcolor{light-blue!47}97.6\% & \cellcolor{light-blue!48}98.1\% &  398.1s & \$1.185 \\
OpenCode    & DeepSeek V4 Pro & high   &  31 &  0 & 12 &  11 & \cellcolor{light-blue!50}\textbf{100.0\%} & \cellcolor{light-blue!21}73.8\% & \cellcolor{light-blue!27}79.6\% & 1366.9s & \$0.127 \\
\midrule
\rowcolor{black!10}\multicolumn{12}{l}{\emph{EquivaFormulation \cite{zhai2025a} (\texttt{p2}, \texttt{p3})}} \\
\midrule
Claude Code & Opus 5          & medium &  10 &  0 &  8 &   0 & \cellcolor{light-blue!50}\textbf{100.0\%} & \cellcolor{light-blue!50}\textbf{100.0\%} & \cellcolor{light-blue!50}\textbf{100.0\%} &  229.8s & \$0.628 \\
Codex       & GPT-5.6 Sol     & medium &  10 &  0 &  8 &   0 & \cellcolor{light-blue!50}\textbf{100.0\%} & \cellcolor{light-blue!50}\textbf{100.0\%} & \cellcolor{light-blue!50}\textbf{100.0\%} &  253.6s & \$0.722 \\
OpenCode    & DeepSeek V4 Pro & high   &  10 &  0 &  8 &   0 & \cellcolor{light-blue!50}\textbf{100.0\%} & \cellcolor{light-blue!50}\textbf{100.0\%} & \cellcolor{light-blue!50}\textbf{100.0\%} &  370.7s & \$0.031 \\
\midrule
\rowcolor{black!10}\multicolumn{12}{l}{\emph{EvoCut \cite{yazdani2025} (\texttt{p6}, \texttt{p8}--\texttt{p12})}} \\
\midrule
Claude Code & Opus 5          & medium &  25 &  0 &  3 &   0 & \cellcolor{light-blue!50}\textbf{100.0\%} & \cellcolor{light-blue!50}\textbf{100.0\%} & \cellcolor{light-blue!50}\textbf{100.0\%} &  413.4s & \$1.116 \\
Codex       & GPT-5.6 Sol     & medium &  25 &  0 &  3 &   0 & \cellcolor{light-blue!50}\textbf{100.0\%} & \cellcolor{light-blue!50}\textbf{100.0\%} & \cellcolor{light-blue!50}\textbf{100.0\%} &  478.1s & \$1.436 \\
OpenCode    & DeepSeek V4 Pro & high   &  17 &  0 &  3 &   8 & \cellcolor{light-blue!50}\textbf{100.0\%} & \cellcolor{light-blue!14}68.0\% & \cellcolor{light-blue!18}71.4\% & 1874.2s & \$0.171 \\
\midrule
\rowcolor{black!10}\multicolumn{12}{l}{\emph{\citet{ferchtandiker2025} (\texttt{p13}--\texttt{p20})}} \\
\midrule
Claude Code & Opus 5          & medium &   7 &  0 &  1 &   0 & \cellcolor{light-blue!50}\textbf{100.0\%} & \cellcolor{light-blue!50}\textbf{100.0\%} & \cellcolor{light-blue!50}\textbf{100.0\%} &  807.5s & \$2.645 \\
Codex       & GPT-5.6 Sol     & medium &   6 &  0 &  1 &   1 & \cellcolor{light-blue!50}\textbf{100.0\%} & \cellcolor{light-blue!34}85.7\% & \cellcolor{light-blue!36}87.5\% &  443.6s & \$1.348 \\
OpenCode    & DeepSeek V4 Pro & high   &   4 &  0 &  1 &   3 & \cellcolor{light-blue!50}\textbf{100.0\%} & \cellcolor{light-blue!2}57.1\% & \cellcolor{light-blue!8}62.5\% & 1833.0s & \$0.187 \\
\bottomrule
\end{tabular}
\end{table}

%% file: tables/models_p12.tex
\begin{table}[ht]
\centering
\captionsetup{font=footnotesize, labelfont=bf}
\caption{Evaluation of \texttt{FLARE-NL} across LLM models and reasoning effort levels. Results show \texttt{FLARE-NL} evaluated on the FormulationBench dataset TSP problem. Metric cells report mean $\pm$ std. dev. across 3 runs; TP/FP/TN/FN reports totals. Every model achieves perfect accuracy with reasoning enabled; performance without reasoning varies widely across model families. To reduce costs, we only evaluate the leading frontier reasoning model from each provider in the ablation study (Table \ref{tab:ablation}).}
\label{tab:models-p12}
\scriptsize
\setlength{\tabcolsep}{4pt}
\begin{tabular}{llrrrrrrrrr}
\toprule
\textbf{Model} & \textbf{Effort} & \textbf{TP} & \textbf{FP} & \textbf{TN} & \textbf{FN} & \textbf{Precision} & \textbf{Recall} & \textbf{Accuracy} & \textbf{Avg Time} & \textbf{Avg Cost} \\
\midrule
\rowcolor{black!10}\multicolumn{11}{l}{\emph{Reasoning Disabled}} \\
\midrule
Opus 5            &        & 15 & 0 & 9 &  0 & \cellcolor{light-blue!50}\textbf{100.0$\pm$0.0\%} & \cellcolor{light-blue!50}\textbf{100.0$\pm$0.0\%} & \cellcolor{light-blue!50}\textbf{100.0$\pm$0.0\%} & 22.2s  & \$0.062  \\
Sonnet 5          &        & 15 & 0 & 9 &  0 & \cellcolor{light-blue!50}\textbf{100.0$\pm$0.0\%} & \cellcolor{light-blue!50}\textbf{100.0$\pm$0.0\%} & \cellcolor{light-blue!50}\textbf{100.0$\pm$0.0\%} & 38.9s  & \$0.042  \\
GPT-5.6 Sol       &        & 15 & 0 & 9 &  0 & \cellcolor{light-blue!50}\textbf{100.0$\pm$0.0\%} & \cellcolor{light-blue!50}\textbf{100.0$\pm$0.0\%} & \cellcolor{light-blue!50}\textbf{100.0$\pm$0.0\%} & 6.9s   & \$0.020  \\
GPT-5.6 Terra     &        & 12 & 0 & 9 &  3 & \cellcolor{light-blue!50}\textbf{100.0$\pm$0.0\%} & \cellcolor{light-blue!25}\underline{80.0$\pm$20.0\%} & \cellcolor{light-blue!34}\underline{87.5$\pm$12.5\%} & 5.2s   & \$0.008  \\
GPT-4.1           &        &  0 & 0 & 9 & 15 & \cellcolor{light-blue!0}---                       & \cellcolor{light-blue!0}0.0$\pm$0.0\%                & \cellcolor{light-blue!0}37.5$\pm$0.0\%              & 6.5s   & \$0.008  \\
DeepSeek V4 Pro   &        &  7 & 0 & 9 &  8 & \cellcolor{light-blue!50}\textbf{100.0$\pm$0.0\%} & \cellcolor{light-blue!0}46.7$\pm$11.5\%              & \cellcolor{light-blue!8}66.7$\pm$7.2\%              & 6.4s   & \$0.001  \\
DeepSeek V4 Flash &        &  7 & 1 & 8 &  8 & \cellcolor{light-blue!40}\underline{91.7$\pm$14.4\%} & \cellcolor{light-blue!0}46.7$\pm$23.1\%           & \cellcolor{light-blue!3}62.5$\pm$12.5\%             & 5.0s   & \$0.0004 \\
\midrule
\rowcolor{black!10}\multicolumn{11}{l}{\emph{Reasoning Enabled}} \\
\midrule
Opus 5            & medium & 15 & 0 & 9 &  0 & \cellcolor{light-blue!50}\textbf{100.0$\pm$0.0\%} & \cellcolor{light-blue!50}\textbf{100.0$\pm$0.0\%} & \cellcolor{light-blue!50}\textbf{100.0$\pm$0.0\%} & 16.5s  & \$0.049  \\
Sonnet 5          & medium & 15 & 0 & 9 &  0 & \cellcolor{light-blue!50}\textbf{100.0$\pm$0.0\%} & \cellcolor{light-blue!50}\textbf{100.0$\pm$0.0\%} & \cellcolor{light-blue!50}\textbf{100.0$\pm$0.0\%} & 23.4s  & \$0.029  \\
GPT-5.6 Sol       & medium & 15 & 0 & 9 &  0 & \cellcolor{light-blue!50}\textbf{100.0$\pm$0.0\%} & \cellcolor{light-blue!50}\textbf{100.0$\pm$0.0\%} & \cellcolor{light-blue!50}\textbf{100.0$\pm$0.0\%} & 12.8s  & \$0.027  \\
GPT-5.6 Terra     & medium & 15 & 0 & 9 &  0 & \cellcolor{light-blue!50}\textbf{100.0$\pm$0.0\%} & \cellcolor{light-blue!50}\textbf{100.0$\pm$0.0\%} & \cellcolor{light-blue!50}\textbf{100.0$\pm$0.0\%} & 10.6s  & \$0.012  \\
DeepSeek V4 Pro   & high   & 15 & 0 & 9 &  0 & \cellcolor{light-blue!50}\textbf{100.0$\pm$0.0\%} & \cellcolor{light-blue!50}\textbf{100.0$\pm$0.0\%} & \cellcolor{light-blue!50}\textbf{100.0$\pm$0.0\%} & 65.8s  & \$0.005  \\
DeepSeek V4 Flash & high   & 15 & 0 & 9 &  0 & \cellcolor{light-blue!50}\textbf{100.0$\pm$0.0\%} & \cellcolor{light-blue!50}\textbf{100.0$\pm$0.0\%} & \cellcolor{light-blue!50}\textbf{100.0$\pm$0.0\%} & 46.3s  & \$0.002  \\
\bottomrule
\end{tabular}
\end{table}